\documentclass[13pt]{article}
\usepackage{latexsym}
\usepackage{geometry}
\usepackage{graphicx}
\usepackage{amsmath, amssymb, amsthm}
\usepackage{booktabs}
\usepackage{algorithm}
\usepackage{algorithmic}

\usepackage[utf8]{inputenc} 
\usepackage[T1]{fontenc}    

\usepackage{url}
\usepackage{natbib}

\usepackage{appendix}

\usepackage{amsfonts}
\usepackage{multirow}
\usepackage{multicol}

\usepackage{color}

\usepackage{xcolor,colortbl}
\definecolor{LightCyan}{rgb}{0.88,1,1}

\usepackage{subfigure}
\usepackage{hyperref}
\usepackage{tcolorbox}

\usepackage{subcaption}
\usepackage{booktabs}
\usepackage{tabularx}
\usepackage{array}
\usepackage{amsmath}
\usepackage[T1]{fontenc}
\usepackage{newtxtext,newtxmath}

\usepackage{color}
\usepackage{xcolor,colortbl}

\def\R{\mathbb{R}}
\def\E{\mathbb{E}}

\def\F{{\rm F}}

\newcommand{\tr}{\mbox{tr}}

\newtheorem{theorem}{Theorem}
\newtheorem{lemma}{Lemma}

\newtheorem{definition}{Definition}
\newtheorem{assumption}{Assumption}
\newtheorem{remark}{Remark}

\begin{document}
\title{ MiMuon: Mixed Muon Optimizer with Improved  Generalization for Large Models }
\author{Feihu Huang\thanks{Feihu Huang is with College of Computer Science and Technology,
Nanjing University of Aeronautics and Astronautics, Nanjing, China;
and also with MIIT Key Laboratory of Pattern Analysis and Machine Intelligence, Nanjing, China. Email: huangfeihu2018@gmail.com}, \
Yuning Luo\thanks{Yuning Luo is with College of Design and Engineering, National University of Singapore, Singapore.}, \
Songcan Chen\thanks{Songcan Chen is with College of Computer Science and Technology,
Nanjing University of Aeronautics and Astronautics, Nanjing, China;
and also with MIIT Key Laboratory of Pattern Analysis and Machine Intelligence, Nanjing, China.}
 }

\date{}
\maketitle

\begin{abstract}
 Matrix-structured parameters frequently appear in many artificial intelligence models such as large language models. More recently, an efficient Muon optimizer is designed for matrix parameters of large-scale models, and
 shows markedly faster convergence than the vector-wise algorithms. Although some works have begun to study convergence properties (i.e., optimization error) of the Muon optimizer, its generalization properties (i.e., generalization error) is still not established. Thus, in this paper, we study generalization error of the Muon optimizer based on algorithmic stability and mathematical induction, and prove that the Muon has a generalization error of $O\big(\frac{1}{N\kappa^{T}}\big)$, where $N$ is training sample size, and $T$ denotes iteration number, and $\kappa>0$ denotes minimum difference between singular values of gradient estimate. To enhance generalization of the Muon, we propose an effective mixed Muon (MiMuon) optimizer by cautiously using orthogonalization of gradient, which is a hybrid of Muon and momentum-based SGD optimizers. Then we prove that our MiMuon optimizer has a lower generalization error of $O\big(\frac{1}{N}\big)$ than $O\big(\frac{1}{N\kappa^{T}}\big)$ of Muon optimizer, since $\kappa$ generally is very small. Meanwhile, we also studied the convergence properties of our MiMuon algorithm, and prove that our MiMuon algorithm has the same convergence rate of $O(\frac{1}{T^{1/4}})$ as the Muon algorithm. Some numerical experimental results on training large models including Qwen3-0.6B and YOLO26m demonstrate efficiency of the MiMuon optimizer.
\end{abstract}

\section{Introduction}
\vspace*{-6pt}
Efficient optimization algorithms~\cite{bottou2018optimization} have been widely received attention in training or inferring large-scale Artificial Intelligence (AI) models such as deep learning models.
Over the past decade, training or inferring large-scale AI models has been dominated by
the vector-wise optimization methods including Stochastic Gradient Descent (SGD)~\citep{robbins1951stochastic}, momentum-based SGD (SGDM)~\citep{sutskever2013importance}, Adam~\cite{kingma2014adam}
and AdamW~\cite{loshchilov2017decoupled} algorithms. Although these vector-wise optimization
algorithms are cornerstone in training AI models, they ignore the inherent matrix structure of
parameters in AI models such as the convolutional layers in Convolutional Neural Networks (CNNs)~\cite{krizhevsky2017imagenet}, and the query, key,
and value projections in transformers~\cite{vaswani2017attention}, which results in suboptimal convergence efficiency.
Recently, some works have begun to study the matrix-wise optimizer to solve the matrix-structured models. For example, the Shampoo~\cite{gupta2018shampoo} optimizer and its variants~\cite{duvvuri2024combining} have been developed by applying left and right preconditioning matrices, which
have demonstrated comparable performance to the popular vector-wise optimizers.

More recently, a typical matrix-wise optimizer, i.e., Muon~\cite{jordanmuon} has been proposed to train the large-scale  matrix-structured models, which shows markedly faster convergence than the vector-valued algorithms
in training matrix parameters of large models~\cite{liu2025muon}. A key feature of Muon optimizer uses orthogonalized gradient to update the matrix parameters. Subsequently, \cite{su2025isotropic} studied the orthogonalization property of Muon, which shows that the orthogonalized gradient becomes optimal for the
isotropic curvature model when the curvature exhibits a phase transition in growth.
\cite{ma2026preconditioning} proved that simplified Muon converges linearly
with iteration complexities independent of the relevant condition number, provably
outperforming gradient descent and Adam.
Meanwhile, some effective variants of Muon optimizer~\cite{sfyraki2025lions,riabinin2025gluon,huang2025limuon,he2025low,liu2025mars,qian2025muon,lang2026powering} have been proposed. For example, \cite{riabinin2025gluon} designed a stochastic adaptive layer-wise Muon optimizer. \cite{refael2025sumo,huang2025limuon} proposed a class of memory-efficient Muon algorithms by using low-rank approximated momentum. \cite{lang2026powering} proposed a zeroth-order variant of Muon optimizer for black-box optimization.

\begin{table*}
	\centering
	\caption{ \textbf{Generalization error} comparison of our MiMuon optimizer and other optimizers for nonconvex optimization. Here $N$ denotes the training dataset size, and $\kappa>0$ generally is very small.  }
	\label{tab:1}
	\resizebox{0.78\textwidth}{!}{
		\begin{tabular}{c|c|c|c}
			\hline
			\textbf{Algorithm} & \textbf{Reference} & \textbf{Generalization Error} & \textbf{Orthogonalization}
			\\ \hline
			SGD  & \cite{hardt2016train} & $O(\frac{1}{N})$  &  \\   \hline
			SGDM  & \cite{ramezani2024generalization} & $O(\frac{1}{N})$  &  \\  \hline
			Muon  & \cite{jordanmuon} & $O\big(\frac{1}{N\kappa^{T}}\big)$  & $\checkmark$  \\  \hline
			MiMuon & Ours &  \textcolor{red}{$O(\frac{1}{N})$} & \textcolor{red}{$\checkmark$}  \\  \hline
		\end{tabular}
	}
\end{table*}

Recently, the convergence properties of Muon optimizer have received attention at some works.
For example, \cite{shen2025convergence,li2025note,chang2025convergence,kovalev2025understanding}
proved that the Muon optimizer converge to
find a stationary solution of nonconvex optimization under the standard smoothness condition.
\cite{pethick2025training,sfyraki2025lions} provided convergence analysis of the Muon optimizer based on the stochastic Frank-Wolfe framework~\cite{hazan2016variance}, and proved that the Muon converge to
find a stationary solution of nonconvex stochastic constrained optimization.
\cite{sfyraki2025lions,huang2025limuon,liu2025mars,qian2025muon} analyzed the convergence properties of the variance reduced Muon algorithm. Subsequently, \cite{riabinin2025gluon,pethick2025generalized,huang2025limuon} have studied convergence properties of the Muon optimizer and its variants under generalized smoothness condition.
\cite{sfyraki2025lions,yu2026sign} studied the convergence properties of Muon algorithm under heavy-tailed noise setting. More recently, \cite{kim2026convergence} provided the convergence analysis for Muon algorithm with Newton-Schulz iteration.

So far, the existing works only focus on analyzing convergence properties (i.e., optimization error) of the Muon optimizer and its variants, but its generalization properties (i.e., generalization error) is still unexplored.
More recently, although \cite{vasudeva2025generalization} proved that spectral gradient descent (SpecGD) dominates gradient descent when learning linear models on class-imbalanced data,
generalization of the SpecGD still is missing.
It is well known that optimization error~\cite{bottou2018optimization} only shows
how quickly the empirical training loss decreases. In fact, the ultimate goal of AI models has a good performance on unseen example. Generalization error~\cite{shalev2010learnability,hardt2016train,zhang2023mathematical} measures
the gap between training loss and population loss, and a small generalization error shows strong performance on unseen example. However, it generally can remain large even
when optimization error is small, leading to overfitting of AI models.
Recently, although optimization error of the Muon optimizer has been studied, its rigorous generalization analysis is still missing. To fill this gap, in the paper,
we study generalization error of the Muon optimizer, and
propose an effective mixed Muon optimizer by cautiously using orthogonalization of gradient to improve generalization of the Muon optimizer.
\vspace*{-6pt}
\subsection*{Contributions}
\vspace*{-6pt}
In the paper, our main contributions are given as follows:
\begin{itemize}
	\item[1)] We study generalization error of the Muon optimizer based on algorithmic stability and mathematical induction, and prove that the Muon has a generalization error of $O\big(\frac{1}{N\kappa^{T}}\big)$, where $N$ denotes training sample size, and $T$ is iteration number, and $\kappa>0$ denotes minimum difference between singular values of gradient estimator that affects stability of Muon algorithm.
	\item[2)] To improve generalization of the Muon, we propose an effective mixed Muon (MiMuon) optimizer by cautiously using orthogonalization of gradient. Then we prove that our MiMuon optimizer has a lower generalization error of $O\big(\frac{1}{N\kappa^{T}}\big)$ than $O\big(\frac{1}{N}\big)$ of Muon optimizer, since $\kappa>0$ generally is very small.
	\item[3)] We study the convergence properties of our MiMuon algorithm, and prove that our MiMuon algorithm has the same convergence rate of $O(\frac{1}{T^{1/4}})$ as the Muon algorithm~\cite{li2025note,chang2025convergence}.
\end{itemize}
From the above Table~\ref{tab:1}, we can find that the Muon optimizer still
has a higher generalization error of $O\big(\frac{1}{N\kappa^{T}}\big)$
than $O\big(\frac{1}{N}\big)$ of both the SGD~\citep{hardt2016train}
and SGDM~\citep{ramezani2024generalization},
since $\kappa$ is generally very small.
This implies that the Muon optimizer
sometimes has a worse generalization than that of SGD and SGDM. Meanwhile,
\cite{dragutinovic2026use} shown that the Muon sometimes
has some disadvantages due to removing its simplicity bias.
However, our MiMuon is a hybrid of Muon and SGDM optimizers,
which incorporates the advantages of both Muon and SGDM.
From Table~\ref{tab:1}, our MiMuon algorithm has the same generalization error of $O\big(\frac{1}{N}\big)$ as
that of both SGD and SGDM. More recently, YOLO26~\cite{sapkota2025yolo26} introduced an effective MuSGD optimizer to obtain stable convergence across diverse datasets,
which the Muon update directly plus SGD update at each iteration.
\vspace*{-6pt}
\section*{Notations}
$\R^+$ denotes non-negative real number set. Let $[N]=\{1,2,\cdots,N\}$.
For matrices $X,Y\in \R^{m\times n}$, Frobenius inner product defined
as $\langle X,Y\rangle = \tr(X^\top Y)$ and Frobenius norm defined as $\|X\|_\F=\sqrt{\tr(X^\top X)}$,
where $\tr(\cdot)$ denotes the trace of a square matrix.
Given a matrix $X\in \R^{m\times n}$, $\|X\|_*$ denotes its nuclear norm, and
$\|X\|_2$ denotes its spectral norm.
$X^\top$ denotes the transpose of matrix $X$.
$a_t=O(b_t)$ denotes that $a_t \leq c b_t$ for some constant $c>0$.

\section{Preliminaries}
\subsection{Problem Setup}
In the paper, we consider solving the following
stochastic matrix-structured optimization problem
\begin{align}\label{eq:p1}
	\min_{W\in \R^{m\times n}} F(W) = \E_{\xi\sim \mathcal{D}}[f(W;\xi)],
\end{align}
where $f(W;\xi): \R^{m\times n}\rightarrow \R^+$ denotes a loss on the sample $\xi\sim \mathcal{D}$,
which is possibly  non-convex function. $W\in \R^{m\times n}$ denotes a parameter matrix, and $\xi$ is a random variable drawn
some fixed but unknown distribution $\mathcal{D}$. Here $F(W) = \E_{\xi\sim \mathcal{D}}[f(W;\xi)]$
denotes a population loss (risk) of AI tasks such as training large language models.
In practice, the fixed distribution $\mathcal{D}$ is
unknown, so we only have access to a finite set of
training examples $S=\{\xi_1,\xi_2,\cdots,\xi_N\}$ drawn i.i.d. from $\mathcal{D}$.
Then the population risk $F(W)$ can be approximated by
the following empirical risk
\begin{align}
	F_S (W) = \frac{1}{N}\sum_{i=1}^N f(W;\xi_i).
\end{align}

\begin{algorithm}
	\caption{ \textbf{Muon} Algorithm~\cite{jordanmuon}}
	\label{alg:1}
	\begin{algorithmic}[1]
		\STATE \textbf{Input}: $\eta>0$, $\beta \in (0,1]$ and $\lambda\geq 0$;
		\STATE \textbf{Initialize:} $W_0\in \R^{m\times n}$,
		and $M_0=0$;
		\FOR{$t = 1, 2, \ldots, T$}
		\STATE  Draw a sample $\xi_{t} \sim \mathcal{D}$;
		\STATE $M_{t} = \beta\nabla f(W_{t-1};\xi_{t}) + (1-\beta)M_{t-1}$;
		\STATE $(U_{t}, \Sigma_{t}, V_{t}) = \text{SVD}(M_{t})$;  (// using Newton-Schulz iterations \textbf{in practice})
		\STATE $W_{t} = W_{t-1} - \eta (U_{t}V_{t}^\top + \lambda W_{t-1})$;
		\ENDFOR
		\STATE \textbf{Output}: $W_T$.
	\end{algorithmic}
\end{algorithm}

\subsection{ Definition of Generalization Error }
In this subsection, we provide a definition of generalization error.
Given a training dataset $S$, we run a
randomized algorithm $A$ to minimize the empirical risk to get a model $A(S)$.
Generally, it does not necessarily mean the output model $A(S)$ would have a
good performance on testing examples, which is measured
by the population risk $F(W) =\E_{\xi\sim \mathcal{D}}[f(W;\xi)]$.
Here, we are interested in the excess
population risk $F(A(S))-F(W^*)$, which measures the relative behavior of the output model as compared
to the best model $$W^*=\mathop{\arg\min}_{W\in \R^{m\times n}}F(W).$$
In fact, we can decompose this excess population risk into the following formation
\begin{align}
	F(A(S))-F(W^*)= \underbrace{F(A(S)) - F_S(A(S))}_{(i)}  + \underbrace{F_S(A(S)) - F_S(W^*)}_{(ii)} + F_S(W^*) -F(W^*), \label{eq:ep}
\end{align}
where the term $(i)$ $F(A(S)) - F_S(A(S))$ is generalization error (generalization gap), which measures
the gap between training loss and population loss;
and the term $(ii)$ $F_S(A(S)) - F_S(W^*)$ is optimization error, which quantifies how well the algorithm minimizes the empirical risk. Taking expectation on the inequality~(\ref{eq:ep}) with random algorithm $A$ and
training dataset $S$, we have
\begin{align}
	\E_{A,S} [F(A(S)) -F(W^*) ]= \E_{A,S} [ F(A(S)) - F_S(A(S))]  + \E_{A,S} [F_S(A(S)) - F_S(W^*)], \nonumber
\end{align}
since $\E_{A,S}[F_S(W^*) -F(W^*)]=0$, i.e., $W^*$
is independent of $A$ and $S$.

\subsection{ Algorithmic Stability }
In this subsection, we provide some definitions and properties of algorithmic stability.
For notational simplicity, let $S=\{\xi_1,\xi_2,\cdots,\xi_N\}$ and $\tilde{S}=\{\tilde{\xi}_1,\tilde{\xi}_2,\cdots,\tilde{\xi}_N\}$ be two
independent datasets drawn from the same distribution $\mathcal{D}$. Here we denote
the dataset $S^{(i)}=\{\xi_1,\xi_2,\cdots,\tilde{\xi}_i,\cdots,\xi_N\}$ by replacing the $i$-th example $\xi_i$ with an independent sample $\tilde{\xi}_i$ for any $i\in [N]$.

\begin{definition} (\textbf{Stability} in Loss Function)
	Let $A$ be a random algorithm and $A(S)$ denote the output of the algorithm $A$ run on dataset $S$. We say the random algorithm $A$ is $\epsilon$-uniform stable if for any $S$ and $S(i)$,
	\begin{align}
		\sup_{\xi} \E_A [f(A(S),\xi)-f(A(S^{(i)}),\xi)] \leq \epsilon.
	\end{align}
\end{definition}

\begin{lemma} \label{lem:gs}
	(\textbf{Generalization via Stability} in Function Values)~\cite{shalev2010learnability,hardt2016train}.
	Let algorithm $A$ be $\epsilon$-uniformly stable in function
	values. Then we have
	$$ |\E_{A,S}[F(A(S)) - F_S(A(S))]|\leq \epsilon.$$
\end{lemma}

\begin{algorithm}
	\caption{ \textbf{MiMuon} Algorithm}
	\label{alg:2}
	\begin{algorithmic}[1]
		\STATE \textbf{Input}: $\eta>0$, $\beta \in (0,1]$, $\lambda\geq 0$ and $\tau>0$;
		\STATE \textbf{Initialize:} $W_0\in \R^{m\times n}$
		and $M_0=0$;
		\FOR{$t = 1, 2, \ldots, T$}
		\STATE  Draw a sample $\xi_{t} \sim \mathcal{D}$;
		\STATE $M_{t} = \beta\nabla f(W_{t-1};\xi_{t}) + (1-\beta)M_{t-1}$;
		\STATE $(U_{t}, \Sigma_{t}, V_{t}) = \text{SVD}(M_{t})$;  (// using Newton-Schulz iterations \textbf{in practice}) 
		\IF {$\min_{i\neq j\in S_t} |\Sigma_{t,ii}-\Sigma_{t,jj}| \geq \tau$ with $S_t=\{i|\Sigma_{t,ii} \neq 0, 1\leq i\leq d\}$  (// using $\|M_t\|_\F \geq \tau$ \textbf{in practice}) } 
		\STATE $W_{t} = W_{t-1} - \eta (U_{t}V_{t}^\top+\lambda W_{t-1})$;
		\ELSE
		\STATE $W_{t} = W_{t-1} - \eta (M_t+\lambda W_{t-1})$;
		\ENDIF
		\ENDFOR
		\STATE \textbf{Output}: $W_T$.
	\end{algorithmic}
\end{algorithm}

\begin{figure}[ht]
	\centering
	\subfigure[SGDM]{\includegraphics[width=0.30\textwidth]{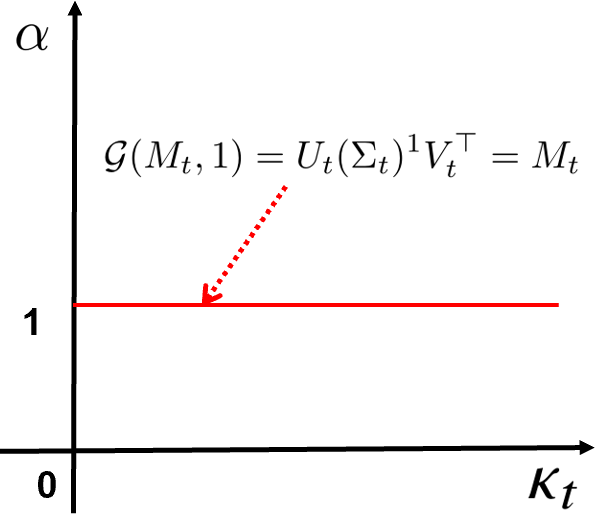}}
	\hfill
	\subfigure[Muon]{\includegraphics[width=0.30\textwidth]{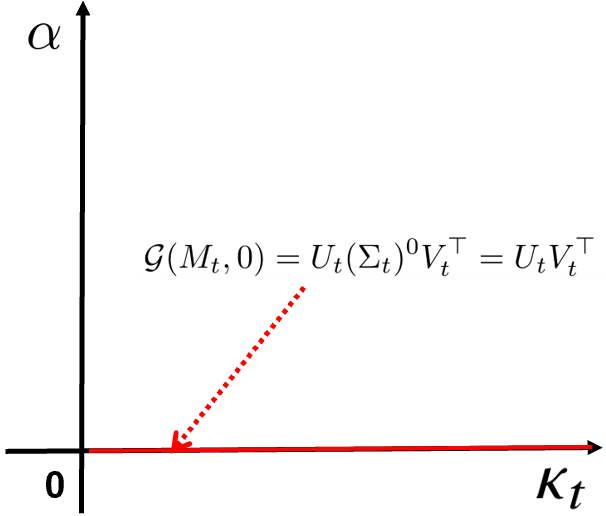}}
	\hfill
	\subfigure[MiMuon]{\includegraphics[width=0.31\textwidth]{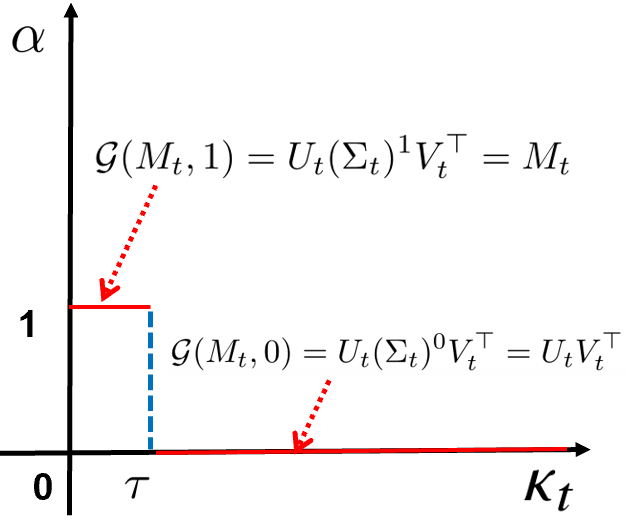}}
	\hfill
	\caption{ Illustration of different
		gradient mapping $\mathcal{G}(M_t,\alpha)$. Here $\kappa_t=\min_{i\neq j} |\sigma_i(M_t)-\sigma_j(M_t)|$.}
	\label{fig:1}
\end{figure}

\section{ Our MiMuon Optimzier }
In this section, we propose an effective mixed Muon (MiMuon) optimizer by cautiously using
orthogonalization of gradient based on the standard Muon optimizer.
We first review the Muon~\cite{jordanmuon} shown in Algorithm~\ref{alg:1}.
At line 6 of Algorithm~\ref{alg:1}, we can use the Newton-Schulz iterations~\cite{jordanmuon,bernstein2024old}
to instead of exact SVD step to approximate the orthogonalization process.
In Algorithm~\ref{alg:1}, 
we can set $\lambda=0$, $\mu = \frac{1-\beta}{\beta}$ and $\eta = \frac{\eta}{\beta}$, so form of Algorithm~\ref{alg:1} is the same as Muon algorithm given in~~\cite{jordanmuon}.

Algorithm~\ref{alg:2} shows an algorithmic framework of our MiMuon optimizer.
When updating variable $W_t$ in Algorithm~\ref{alg:2}, we cautiously uses orthogonalization of gradient estimator $M_t$. If $\min_{i\neq j\in S_t} |\Sigma_{t,ii}-\Sigma_{t,jj}| \geq \tau$ with $S_t=\{i|\Sigma_{t,ii} \neq 0, 1\leq i\leq d\}$, we use the orthogonalized gradient descent to update variable $W_{t}$ as follows:
\begin{align} \label{eq:5}
	W_{t} = (1-\eta\lambda)W_{t-1} - \eta U_{t}V_{t}^\top,
\end{align}
where $\eta>0$ and $0\leq \lambda <\frac{1}{\eta} $.
Otherwise, we the gradient descent to update variable $W_{t}$ as follows:
\begin{align} \label{eq:6}
	W_{t} = (1-\eta\lambda)W_{t-1} - \eta M_t.
\end{align}

Here we define a useful gradient mapping $\mathcal{G}(\cdot,\cdot): \R^{m\times n}\times \R \rightarrow \R^{m\times n}$ and set $\mathcal{G}(M_t,\alpha)= U_{t}(\Sigma_{t})^\alpha V_{t}^\top$, where
$(\Sigma_{t})^\alpha$ denotes an element-wise power operation for non-zero elements of matrix $\Sigma_{t}$. Then
we can uniformly rewrite the above equations~(\ref{eq:5}) and~(\ref{eq:6}) as follows:
\begin{align} \label{eq:7}
	W_{t} = (1-\eta\lambda)W_{t-1} - \eta\mathcal{G}(M_t,\alpha),
\end{align}
where $\mathcal{G}(M_t,0) = U_{t}(\Sigma_{t})^0V_{t}^\top=U_{t}V_{t}^\top$ when using orthogonalized gradient descent, and
$\mathcal{G}(M_t,1) =U_{t}(\Sigma_{t})^1V_{t}^\top= M_t$ when using gradient descent.

For notational simplicity, let $r=\min(m,n)$ denote rank of gradient estimator $M_t\in \R^{m\times n}$
for all $t\geq1$.
Meanwhile, assume $\{\sigma_i(M_t)\}_{i=1}^r$ denote
the singular-value of matrix $M_t$ and satisfies $\sigma_1(M_t)<\sigma_2(M_t)<\cdots<\sigma_r(M_t)$.
Thus, the line 7 of Algorithm~\ref{alg:2} can be rewritten as $\kappa_t=\min_{i\neq j} |\sigma_i(M_t)-\sigma_j(M_t)|\geq\tau>0$. From Figure~\ref{fig:1}, \textbf{(a)} SGDM uses $\mathcal{G}(M_t,1)= M_t$ with $\alpha=1$; \textbf{(b)} Muon uses $\mathcal{G}(M_t,0)= U_{t}V_{t}^\top$ with $\alpha=0$; \textbf{(c)} Our MiMuon uses $\mathcal{G}(M_t,0)= U_{t}V_{t}^\top$ when $\kappa_t=\min_{i\neq j} |\sigma_i(M_t)-\sigma_j(M_t)|\geq\tau>0$, otherwise $\mathcal{G}(M_t,1)= M_t$. 

Since $\sigma_1(M_t)<\sigma_2(M_t)<\cdots<\sigma_r(M_t)$
and $\kappa_t=\min_{i\neq j} |\sigma_i(M_t)-\sigma_j(M_t)|\geq\tau>0$, we have
$$\tau < \tau+\sigma_1(M_t) \leq \sigma_2(M_t)<\cdots<\sigma_r(M_t)\leq \|M_t\|_\F.$$
When the dimension of matrix $M_t$is large in Algorithm~\ref{alg:2},
we also could use the Newton-Schulz iterations to approximate the orthogonalization process
as in Muon algorithm. Under this case, we can not obtain singular value matrix $\Sigma_t$,
in other words, we also can not obtain
the term $\kappa_t=\min_{i\neq j} |\sigma_i(M_t)-\sigma_j(M_t)|$. Thus, we can use
its upper bound such as $\|M_t\|_\F$ instead of
the term $\kappa_t=\min_{i\neq j} |\sigma_i(M_t)-\sigma_j(M_t)|$ in practice.

\section{Generalization Analysis}
In the section, we provide generalization analysis for both Muon and
our MiMuon algorithms under some mild conditions, respectively.
We first give some assumptions on the problem~(\ref{eq:p1}).

\begin{assumption}[\textbf{Smoothness of Population Function}]\label{ass:s1}
	Population function $F(W)$ is $L$-Frobenius norm Lipschitz smooth, if for any $W, W'\in \R^{m\times n}$, we have
	\begin{align}
		\|\nabla F(W)-\nabla F(W')\|_\F\leq L\|W-W'\|_\F.
	\end{align}
\end{assumption}

\begin{assumption}[\textbf{Lipschitzness}]\label{ass:g}
	Function $f(W;\xi)$ for all $\xi\sim \mathcal{D}$  is $G$-Frobenius norm Lipschitz continuous, if for any $W, W'\in \R^{m\times n}$, we have
	\begin{align}
		\|f(W;\xi)-f(W';\xi)\|_\F \leq G\|W-W'\|_\F.
	\end{align}
\end{assumption}

\begin{assumption}[\textbf{Bounded Variance}]\label{ass:v}
	$\nabla f(W;\xi)$ is an unbiased stochastic estimator of the true gradient $\nabla F(W)$ and has a bounded variance, i.e., $\E[\nabla f(W;\xi)]=\nabla F(W)$ and
	\begin{align}
		\E\|\nabla f(W;\xi)-\nabla F(W)\|_\F^2]\leq \sigma^2. \nonumber
	\end{align}
\end{assumption}
Assumption~\ref{ass:s1} shows smoothness of population function $F(W)$, and is a natural extension of $l_2$-norm smoothness for functions with vector parameters to functions with matrix parameters~\cite{beck2017first}, which is a standard Frobenius norm smoothness used in~\cite{shen2025convergence,sfyraki2025lions,pethick2025training}.
Assumption~\ref{ass:g} provide the Lipschitz continuous of objective function, which is widely used in
generalization analysis~\cite{hardt2016train,lei2020fine,lei2023stability}.
Assumption~\ref{ass:v} shows a standard bounded variance assumption used in stochastic optimization~\cite{bottou2018optimization,shen2025convergence,sfyraki2025lions}.

\begin{theorem} \label{th:1}
	Assume the sequence $\{W_t,M_t\}_{t=0}^T$ is generated
	from Algorithm~\ref{alg:1} on dataset $S=\{\xi_1,\xi_2,\cdots,\xi_N\}$. Let $\{\sigma_i(M_t)\}_{i=1}^r$ denote
	the singular-value of momentum matrix $M_t$ for $t=1,2,\cdots,T$, and let $\kappa_t =  \min_{i\neq j}|\sigma_i(M_t)-\sigma_j(M_t)|>0 $.
	Under the above Assumptions~\ref{ass:s1},~\ref{ass:g} and~\ref{ass:v}, let
	$\eta=O(1)$, $\beta=O(1)$ with $\beta\in (0,1]$ and $\lambda=0(1)$ with $\lambda <\frac{1}{\eta}$, then we have
	\begin{align}
		| \E [F(W_T) - F_S(W_T)] | \leq O\Big(\frac{1}{N\kappa^{T}}\Big), \nonumber
	\end{align}
	where $\kappa = \min_{1\leq t \leq T} \kappa_t$.
\end{theorem}
\begin{remark}
	From the above Theorem~\ref{th:1},
	the Muon algorithm has a generalization error bounded by $O(\frac{1}{N\kappa^{T}})$, where $\kappa$ denotes minimum difference between singular values of gradient estimator. Since the parameter $\kappa>0$ is generally small, it affects generalization of the Muon algorithm.
\end{remark}

\begin{theorem} \label{th:2}
	Assume the sequence $\{W_t,M_t\}_{t=0}^T$ is generated
	from Algorithm~\ref{alg:2} on dataset $S=\{\xi_1,\xi_2,\cdots,\xi_N\}$. Under the Assumptions~\ref{ass:s1},~\ref{ass:g},~\ref{ass:v}, without loss of generality, let $\tau\geq 1$, and let
	$\eta=O(1)$, $\beta=O(1)$ with $\beta\in (0,1]$ and $\lambda=0(1)$ with $\lambda <\frac{1}{\eta}$. When the iteration number is relatively small (i.e., $T=O(1)$) set
	$\eta=O(1)$, otherwise set $\eta=\frac{1}{T}$, then we have
	\begin{align}
		|\E [F(W_T) - F_S(W_T)] | \leq O\Big(\frac{1}{N}\Big). \nonumber
	\end{align}
\end{theorem}
\begin{remark}
	From the above Theorem~\ref{th:2},
	our MiMuon algorithm has a generalization error bounded by $O(\frac{1}{N})$, which reaches
	the same generalization error of $O\big(\frac{1}{N}\big)$ as
	that of both SGD~\citep{hardt2016train}
	and SGDM~\citep{ramezani2024generalization}. Thus,
	our MiMuon incorporates the advantages of both Muon and SGDM.
\end{remark}
\section{Convergence Analysis}
In the section, we provide convergence analysis for
our MiMuon algorithms under some mild conditions.
\begin{assumption} \label{ass:b}
	(\textbf{Feasibility})
	The function $F(W)$ has a lower bounded, i.e.,  $F^* = \inf_{W\in
		\mathbb{R}^{m\times n}} F(W)>-\infty$.
\end{assumption}

\begin{lemma} \label{lem:mb}
	Assume the sequence $\{W_t,M_t\}_{t=0}^T$ is generated from Algorithm~\ref{alg:2}.
	Under the Assumptions~\ref{ass:s1},~\ref{ass:v}, given
	$\|W_0\|_\F \leq \eta\hat{G} $ and $\lambda \leq \frac{1}{2(1+T)\eta\hat{G}}$, we have
	\begin{align}
		& \|W_{t}-W_{t-1}\|_\F \leq \eta\breve{G}, \quad \|W_{t}\|_\F \leq (t+1)\eta\hat{G},  \ \forall t\geq 1\nonumber \\
		& \frac{1}{T+1}\sum_{t=0}^{T} \mathbb{E} \|\nabla F(W_t) - M_{t+1}\|_\F  \leq \frac{\sigma}{\sqrt{T\beta}} +  \frac{\sqrt{2}}{\beta}L\eta\breve{G}  + \sqrt{\beta}\sigma,
	\end{align}
	where $\hat{G}=\max(G,\sqrt{r})$ and $\breve{G}=\hat{G}+\frac{1}{2}$.
\end{lemma}

\begin{theorem} \label{th:3}
	Assume the sequence $\{W_t\}_{t=0}^T$ is generated from Algorithm~\ref{alg:2}. Under the Assumptions~\ref{ass:s1},~\ref{ass:g}~\ref{ass:v}~\ref{ass:b},
	and let $\|W_0\|_\F \leq \eta\hat{G} $ and $\lambda \leq \frac{1}{2(1+T)\eta\hat{G}}$, we have
	\begin{align}
		\frac{1}{T+1} \sum_{t=0}^{T} \E \|\nabla F(W_t)\|_\F & \leq \frac{2(F(W_0)-F^*)}{T\eta}  + 4\sqrt{r}\Big( \frac{\sigma}{\sqrt{T\beta}} +  \frac{\sqrt{2}}{\beta}L\eta\breve{G}  + \sqrt{\beta}\sigma \Big) + L\eta\breve{G}^2.
	\end{align}
\end{theorem}
\begin{remark}
	Given $\eta=O(\frac{1}{T^{\frac{3}{4}}})$ and $\beta=\frac{1}{\sqrt{T}}$, we can obtain
	\begin{align}
		\frac{1}{T+1} \sum_{t=0}^T &\E \|\nabla F(W_t)\|_\F
		\leq \frac{2(F(W_0)-F^*)}{T\eta}  + 4\sqrt{r}\Big( \frac{\sigma}{\sqrt{T\beta}} +  \frac{\sqrt{2}}{\beta}L\eta\breve{G}  + \sqrt{\beta}\sigma \Big) + L\eta\breve{G}^2 \nonumber \\
		& = O(\frac{1}{T^{\frac{1}{4}}}) + O(\frac{1}{T^{\frac{1}{4}}}) + O(\frac{1}{T^{\frac{1}{4}}})+ O(\frac{1}{T^{\frac{1}{4}}}) + O(\frac{1}{T^{\frac{3}{4}}}) = O(\frac{1}{T^{\frac{1}{4}}}).
	\end{align}
	Thus, our MiMuon optimizer has the same convergence rate of $O(\frac{1}{T^{1/4}})$ as the Muon optimizer~\cite{li2025note,chang2025convergence}.
\end{remark}

\begin{figure*}[ht]
	\centering
	\subfigure[Training loss.]{\includegraphics[width=0.43\textwidth]{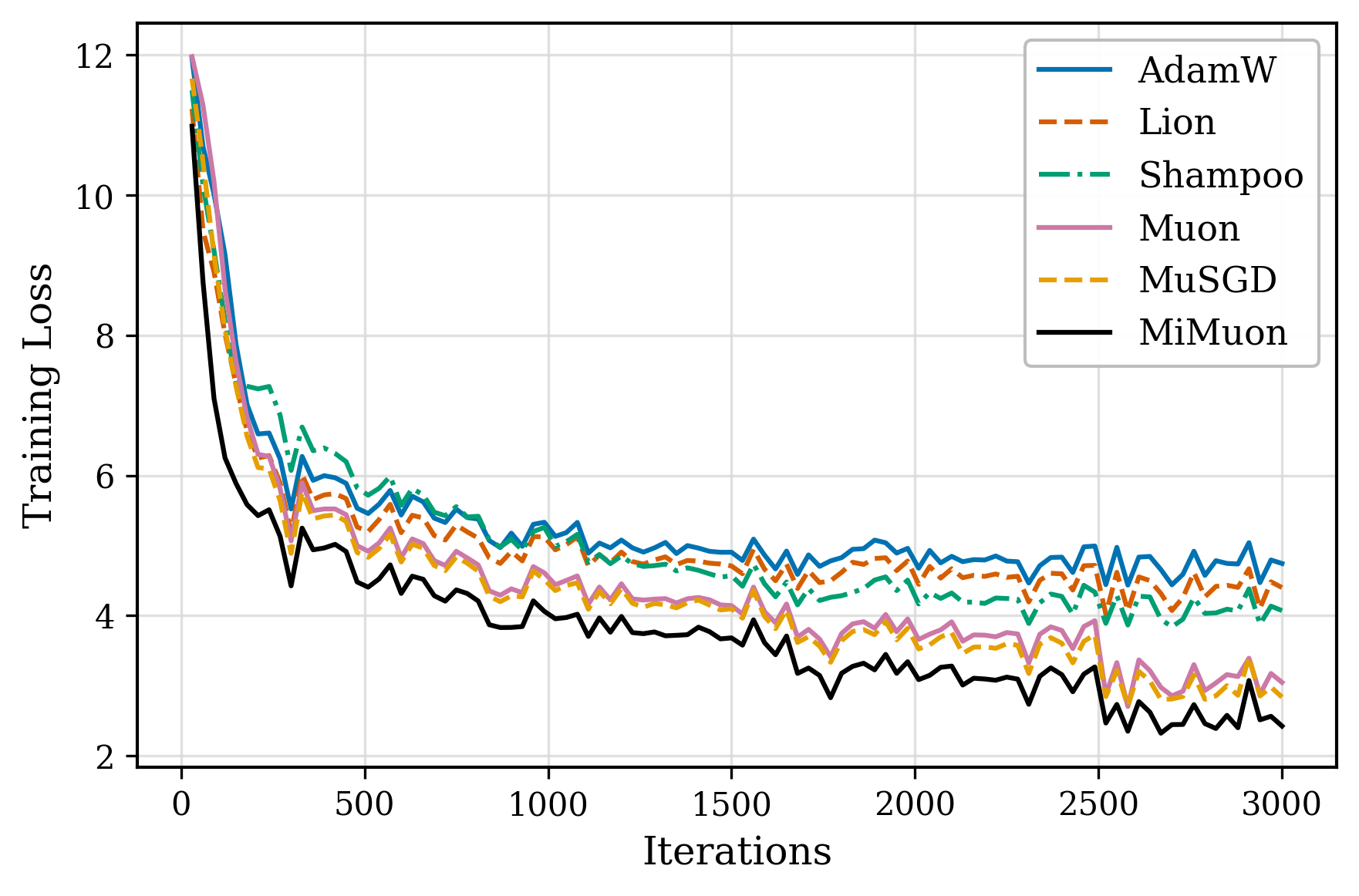}}
	\hfill
	\subfigure[Validation loss.]{\includegraphics[width=0.43\textwidth]{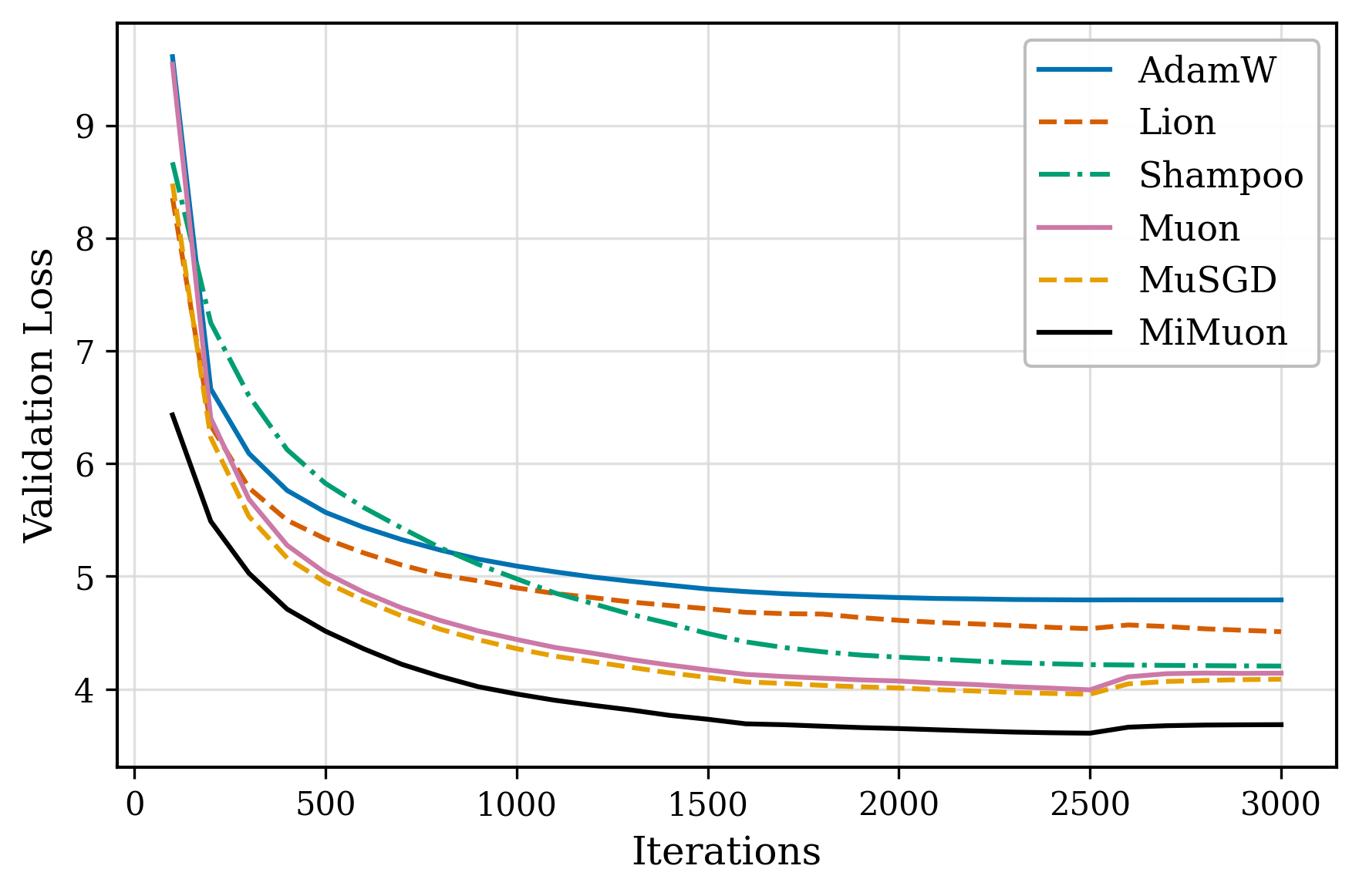}}
	\hfill
	\caption{Optimization behavior on \textbf{Qwen3-0.6B}. }
	\label{fig:qwen_curves}
\end{figure*}


\begin{table}[t]
	\centering
	\caption{ Results on \textbf{Qwen3-0.6B}. The best result is shown in bold and the second-best result is underlined.}
	\label{tab:qwen_results}
	\resizebox{0.62\textwidth}{!}{
		\begin{tabular}{lccc}
			\toprule
			Optimizer & Train loss & Final val. loss & Best val. loss \\
			\midrule
			AdamW   & 4.742 & 4.790 & 4.790 \\
			Lion    & 4.399 & 4.508 & 4.508 \\
			Shampoo & 4.070 & 4.204 & 4.204 \\
			Muon    & 3.046 & 4.141 & 3.993 \\
			MuSGD   & \underline{2.834} & \underline{4.087} & \underline{3.954} \\
			MiMuon  & \textbf{2.422} & \textbf{3.684} & \textbf{3.608} \\
			\bottomrule
		\end{tabular}
	}
\end{table}

\section{Numerical Experiments}
\label{sec:experiments}
In the section, 
we evaluate our MiMuon algorithm on two training workloads: language modeling and object detection. Here, Qwen3-0.6B~\citep{qwen3technicalreport} represents dense Transformer language modeling, where most of the computation is concentrated in large projection matrices. YOLO26m~\citep{sapkota2025yolo26}, by contrast, is a one-stage detector whose objective mixes classification, localization, and distributional box regression. Success on both tasks would indicate that our MiMuon is not merely tuned to a single model family.

In the experiments, we include AdamW~\citep{loshchilov2017decoupled}, Lion~\citep{chen2023symbolic}, Shampoo~\citep{gupta2018shampoo}, Muon~\citep{jordanmuon}, and MuSGD~\citep{sapkota2025yolo26} as reference optimizers. They are used to set the scale of the comparison rather than as separate objects of study. Within each task, all methods use the same model, data split, and training budget. The main hyperparameters and tuning ranges are listed in Appendix~\ref{app:hyperparams}. All experiments were conducted on the NVIDIA A100-SXM4-80GB.

\begin{figure*}[ht]
	\centering
	\subfigure[Box loss.]{\includegraphics[width=0.32\textwidth]{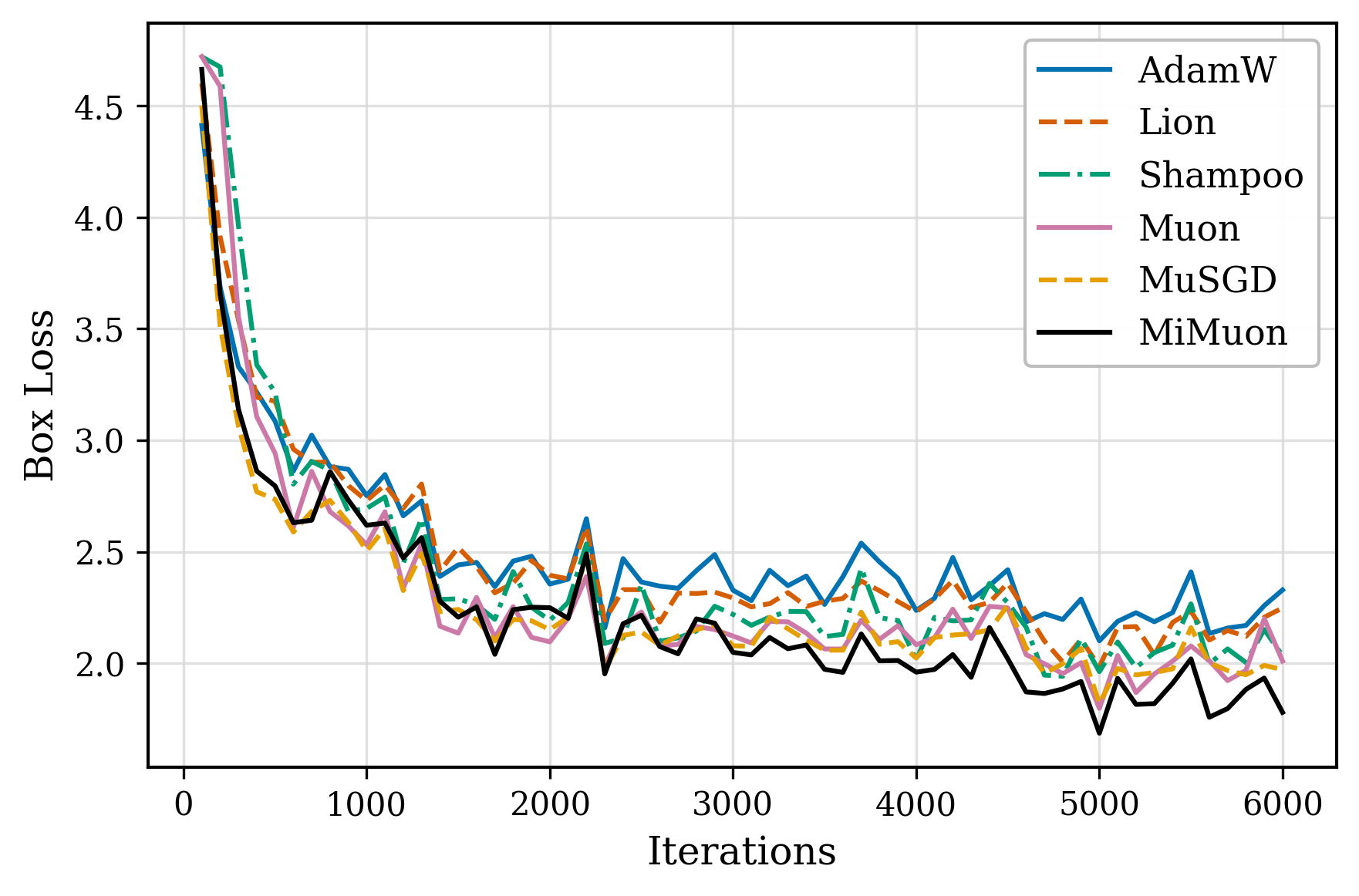}}
	\hfill
	\subfigure[Classification loss.]{\includegraphics[width=0.32\textwidth]{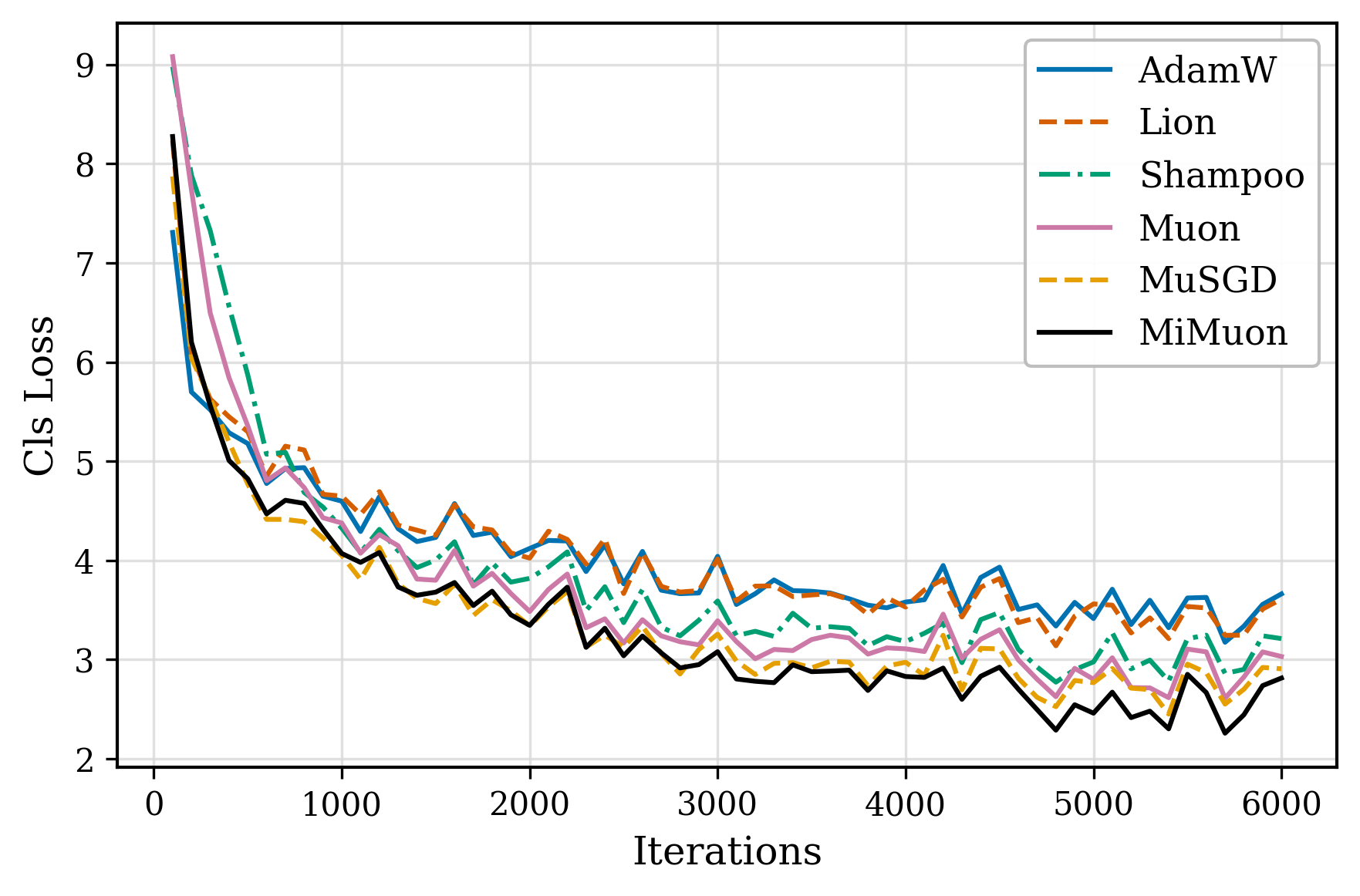}}
	\hfill
	\subfigure[DFL loss.]{\includegraphics[width=0.32\textwidth]{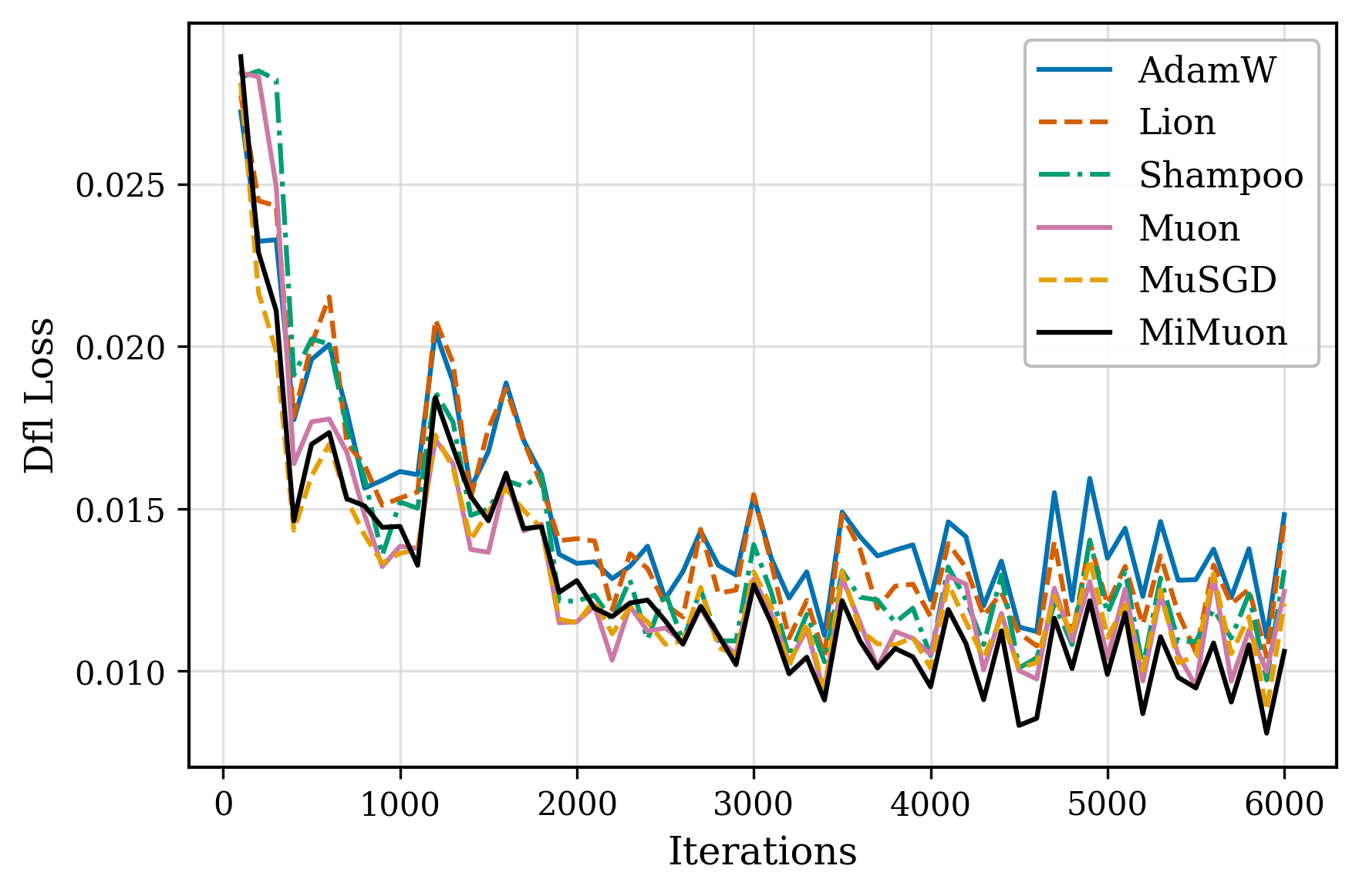}}
	\hfill
	\caption{Training loss components on \textbf{YOLO26m}. }
	\label{fig:yolo_train_losses}
\end{figure*}

\begin{figure*}[ht]
	\centering
	\subfigure[mAP$_{50}$.]{\includegraphics[width=0.43\textwidth]{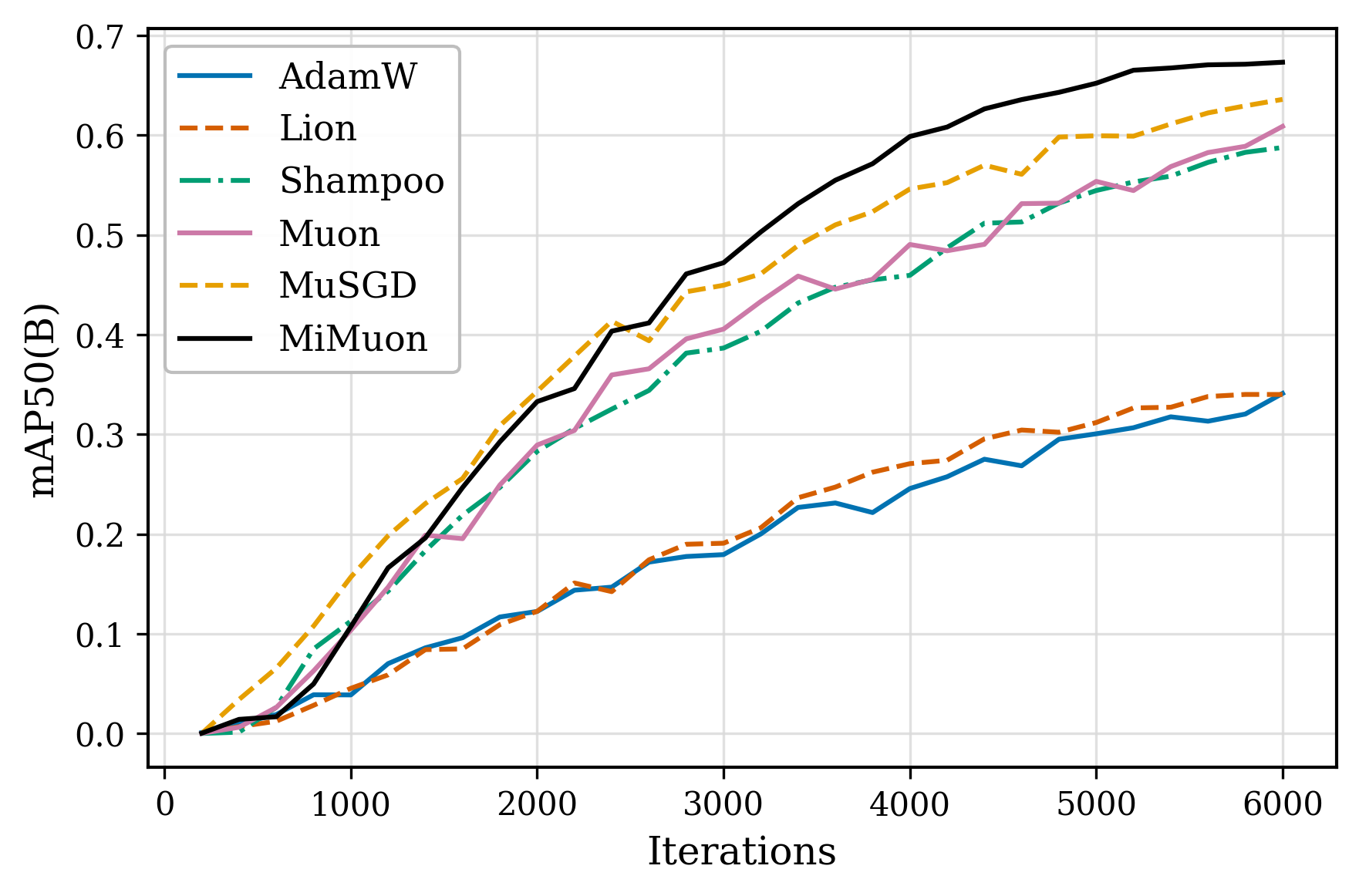}}
	\hfill
	\subfigure[mAP$_{50:95}$.]{\includegraphics[width=0.43\textwidth]{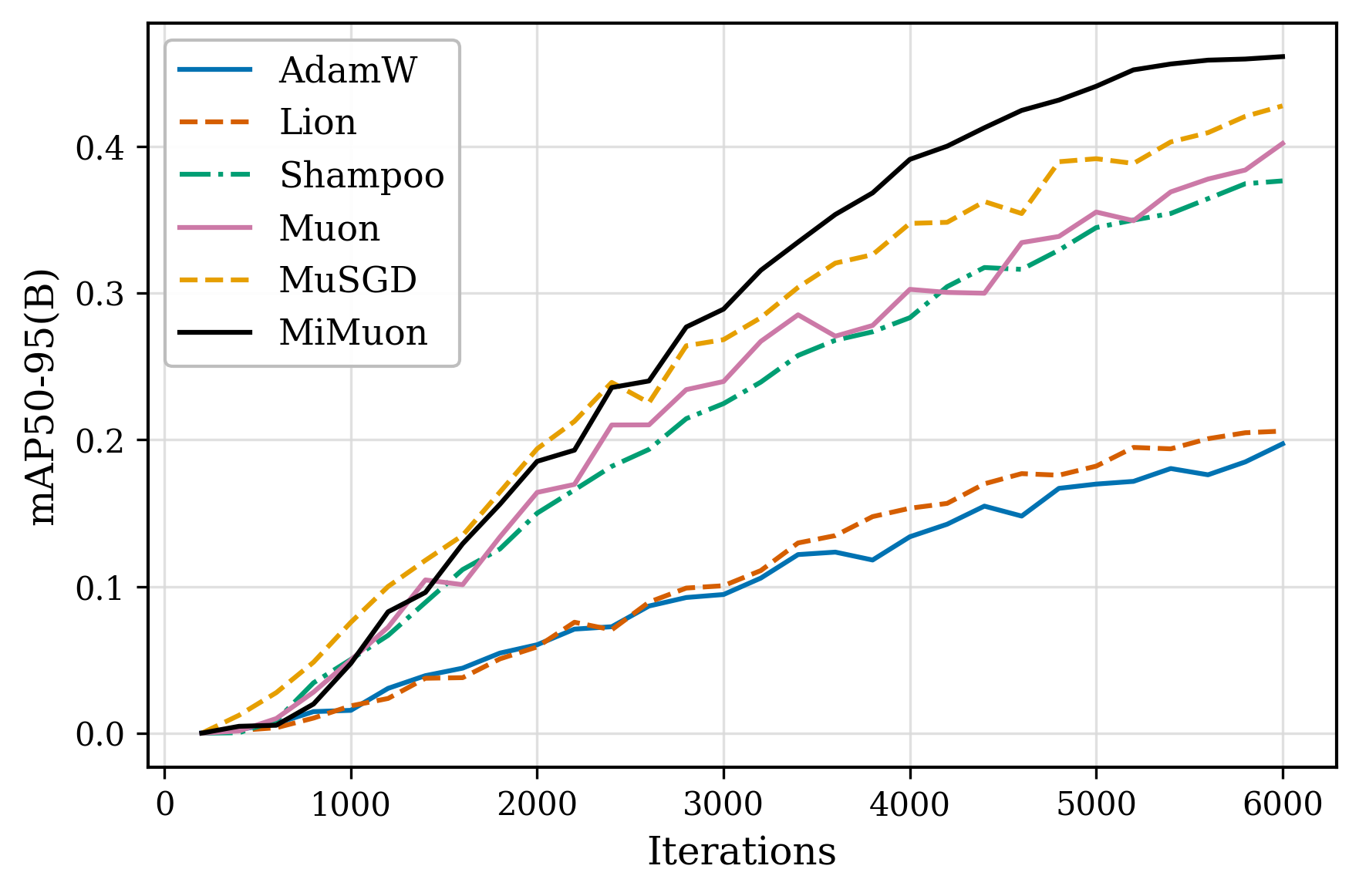}}
	\hfill
	\caption{Validation accuracy (mAP50 and mAP50-95) on \textbf{YOLO26m}. }
	\label{fig:yolo_map}
\end{figure*}

\begin{table*}[t]
	\centering
	\caption{Results on \textbf{YOLO26m}. The best result is shown in bold and the second-best result is underlined. }
	\label{tab:yolo_results}
	\setlength{\tabcolsep}{4.2pt}
	\renewcommand{\arraystretch}{1.12}
	\begin{tabular}{lccccccc}
		\toprule
		Optimizer & Train box & Train cls & Train DFL & Val. box & Val. cls & Val. DFL & Fitness \\
		\midrule
		AdamW   & 2.329 & 3.662 & 0.0148 & 1.859 & 2.501 & 0.0181 & 0.197 \\
		Lion    & 2.248 & 3.611 & 0.0146 & 1.779 & 2.482 & 0.0169 & 0.206 \\
		Shampoo & 2.034 & 3.212 & 0.0131 & 1.614 & 1.836 & 0.0149 & 0.376 \\
		Muon    & 2.010 & 3.031 & 0.0125 & 1.579 & 1.755 & 0.0145 & 0.402 \\
		MuSGD   & \underline{1.972} & \underline{2.907} & \underline{0.0121} & \underline{1.555} & \underline{1.688} & \underline{0.0142} & \underline{0.427} \\
		MiMuon  & \textbf{1.780} & \textbf{2.814} & \textbf{0.0106} & \textbf{1.467} & \textbf{1.553} & \textbf{0.0130} & \textbf{0.461} \\
		\bottomrule
	\end{tabular}
\end{table*}

\subsection{Training Qwen3-0.6B}
\label{subsec:qwen}
We begin with language modeling, where the optimizer must make steady progress through many dense matrix transformations rather than a small number of task-specific layers. For this setting, we use Qwen3-0.6B~\citep{qwen3technicalreport}, a decoder-only Transformer built from repeated self-attention and feed-forward blocks. We train the model from scratch on WikiText-103, a long-form language modeling benchmark with a substantially larger vocabulary and longer contexts than toy character-level corpora. The experiment uses a sequence length of 1024 and a global batch size of 16 sequences, so the comparison emphasizes how efficiently each optimizer uses a limited number of parameter updates.

Figure~\ref{fig:qwen_curves} shows the training and validation loss curves. From these results, our MiMuon descends quickly in the early phase and keeps improving after the easiest loss reduction has already been taken. This behavior is consistent with the design of the method: it can use the matrix-normalized direction when it is productive, while avoiding overly aggressive matrix updates when the local update scale becomes small. The result is both the lowest training loss and the best validation loss among the compared methods.

Table~\ref{tab:qwen_results} gives the same picture in numbers, where lower loss values  is better, while higher fitness is better.  Our MiMuon lowers the final validation loss by 0.403 compared with MuSGD and by 0.457 compared with Muon. Its best validation loss is also the lowest among all methods, so the improvement is not just a single favorable endpoint.

\subsection{Training YOLO26m}
\label{subsec:yolo}
The second experiment moves from token prediction to dense visual recognition. Here the optimizer must balance several loss terms while producing parameters that support accurate localization and classification. We use YOLO26m~\cite{sapkota2025yolo26}, a medium-scale one-stage detector that predicts object categories and bounding boxes from dense multi-scale image features. Unlike language modeling, the optimization target is a weighted combination of box regression, classification, and distribution focal loss (DFL). We train YOLO26m from scratch on Pascal VOC using 640-resolution images and a batch size of 32.

Figure~\ref{fig:yolo_train_losses} shows the three training loss components. Our MiMuon reduces box, classification, and distribution focal losses more effectively across training, and it gives the lowest loss across localization, classification, and distribution focal loss terms. Figure~\ref{fig:yolo_map} reports validation mAP, where the same pattern appears in the final detector quality. Our MiMuon algorithm achieves the best final mAP on both metrics. The gain is therefore not limited to a single scalar loss; it also carries over to the accuracy metric used to judge the detector.

Table~\ref{tab:yolo_results} summarizes the detection results. Our MiMuon gives the lowest training and validation losses for all three YOLO loss components, and it improves the final fitness score by 3.36 points over MuSGD and by 5.92 points over Muon. Taken together, the Qwen and YOLO results show that our MiMuon is not specialized to one architecture or loss function.

\section{Conclusion}
\vspace*{-8pt}
In the paper, we studied  generalization of the Muon optimizer under the nonconvex setting. We proved that the Muon optimizer has a generalization error of $O\big(\frac{1}{N\kappa^{T}}\big)$, where $\kappa$ denotes minimum difference between singular values of gradient estimate. 
To improve generalization of Muon, we proposed an effective mixed Muon (MiMuon) optimizer by cautiously using orthogonalization of gradient, which is a hybrid of Muon and momentum-based SGD optimizers. Then we prove that our MiMuon optimizer has a lower generalization error of $O\big(\frac{1}{N}\big)$ than $O\big(\frac{1}{N\kappa^{T}}\big)$ of Muon optimizer, since $\kappa$ generally is very small. Meanwhile, we  proved that our MiMuon algorithm has the same convergence rate of $O(\frac{1}{T^{1/4}})$ as the Muon algorithm.

\small

\bibliographystyle{plain}

\bibliography{MiMuon}

\newpage

\appendix

\section{Generalization Analysis}
In this subsection, we provide a detailed generalization analysis for both the Muon and
our MiMuon algorithms, respectively.

\begin{theorem} \label{th:t1} (\textbf{Restatement of Theorem~\ref{th:1}})
	Assume the sequence $\{W_t,M_t\}_{t=0}^T$ is generated
	from Algorithm~\ref{alg:1} on dataset $S=\{\xi_1,\xi_2,\cdots,\xi_N\}$. Let $\{\sigma_i(M_t)\}_{i=1}^r$ denote
	the singular-value of momentum matrix $M_t$ for $t=1,2,\cdots,T$, and let $\kappa_t =  \min_{i\neq j}|\sigma_i(M_t)-\sigma_j(M_t)|>0 $.
	Under the above Assumptions~\ref{ass:s1},~\ref{ass:g} and~\ref{ass:v}, let
	$\eta=O(1)$, $\beta=O(1)$ with $\beta\in (0,1]$ and $\lambda=0(1)$ with $\lambda <\frac{1}{\eta}$, then we have
	\begin{align}
		|\E [F(W_T) - F_S(W_T)]| \leq O\Big(\frac{1}{N\kappa^{T}}\Big), \nonumber
	\end{align}
	where $\kappa = \min_{1\leq t \leq T} \kappa_t$.
\end{theorem}

\begin{proof}
	Implementing Algorithm~\ref{alg:1} on datasets $S$ and $S^{(i)}$ with
	the same random index sequence $\{j_t\}_{t=1}^T$, and let $\{W_t\}_{t=1}^T$ and $\{W_t^{(i)}\}_{t=1}^T$ be generated from Algorithm~\ref{alg:1} with $S$ and $S^{(i)}$.
	Let $M_t = \beta\nabla f(W_{t-1};\xi_{t}) + (1-\beta)M_{t-1}$ and $M^{(i)}_t = \beta\nabla f(W_{t-1}^{(i)};\xi_{t}) + (1-\beta)M^{(i)}_{t-1}$ is generated from Algorithm~\ref{alg:1}.
	
	Let $M_t=U_t\Sigma_t V_t^\top$ and $M^{(i)}_t=U^{(i)}_t\Sigma^{(i)}_t (V^{(i)}_t)^\top$, where $\Sigma_t =\mbox{diag}(a_t)$ with $a_t=(\sigma_1(M_t),\cdots,\sigma_r(M_t))$ and
	$\Sigma^{(i)}_t =\mbox{diag}(a^{(i)}_t)$ with $a^{(i)}_t=(\sigma_1(M^{(i)}_t),\cdots,\sigma_r(M^{(i)}_t))$. From Algorithm~\ref{alg:1}, we have
	\begin{align}
		W_{t} = W_{t-1} - \eta (U_tV_t^\top+\lambda W_{t-1}), \quad W^{(i)}_{t} = W^{(i)}_{t-1} - \eta (U^{(i)}_t(V^{(i)}_t)^\top+\lambda W^{(i)}_{t-1}).
	\end{align}
	Then we have
	\begin{align}
		W_{t} - W^{(i)}_{t} = (1-\eta\lambda)(W_{t-1} -W^{(i)}_{t-1}) - \eta (U_tV_t^\top - U^{(i)}_t(V^{(i)}_t)^\top).
	\end{align}
	We can obtain
	\begin{align} \label{eq:w1}
		& \|W_{t} - W^{(i)}_{t}\|_\F \nonumber \\
		& = \|(1-\eta\lambda)(W_{t-1} -W^{(i)}_{t-1}) - \eta (U_tV_t^\top - U^{(i)}_t(V^{(i)}_t)^\top)\|_\F \nonumber \\
		& \leq (1-\eta\lambda)\|W_{t-1} -W^{(i)}_{t-1}\|_\F + \eta \|U_tV_t^\top - U^{(i)}_t(V^{(i)}_t)^\top\|_\F \nonumber \\
		& \leq (1-\eta\lambda)\|W_{t-1} -W^{(i)}_{t-1}\|_\F + \eta \|U_tV_t^\top -U_t(V^{(i)}_t)^\top \|_\F + \eta\|U_t(V^{(i)}_t)^\top - U^{(i)}_t(V^{(i)}_t)^\top\|_\F \nonumber \\
		& \leq (1-\eta\lambda)\|W_{t-1} -W^{(i)}_{t-1}\|_\F + \eta \|V_t -V^{(i)}_t \|_\F\|U_t\|_2 + \eta\|U_t - U^{(i)}_t\|_\F\|V^{(i)}_t\|_2 \nonumber \\
		& = (1-\eta\lambda)\|W_{t-1} -W^{(i)}_{t-1}\|_\F + \eta \|V_t -V^{(i)}_t \|_\F + \eta\|U_t - U^{(i)}_t\|_\F \nonumber \\
		& \leq (1-\eta\lambda)\|W_{t-1} -W^{(i)}_{t-1}\|_\F + \frac{2\sqrt{2}\eta}{\kappa_t} \|M_t - M^{(i)}_t\|_\F \nonumber \\
		& \leq (1-\eta\lambda)\|W_{t-1} -W^{(i)}_{t-1}\|_\F + \frac{2\sqrt{2}\eta}{\kappa} \|M_t - M^{(i)}_t\|_\F ,
	\end{align}
	where the second last inequality holds by the singular-vector version of the Davis-Kahan $\sin\theta$ theorem (Corollary 3 in \cite{yu2015useful}) with $\kappa_t = \min_{i\neq j}|\sigma_i(M_t)-\sigma_j(M_t)|>0$, and
	the last inequality is due to $\kappa = \min_{1\leq t \leq T} \kappa_t$.
	
	Next considering the term $\|M_t - M^{(i)}_t\|_\F$, we have
	\begin{align}
		\E \|M_t - M^{(i)}_t\|_\F & = \|M_t - \nabla F(W_t) +\nabla F(W_t) -\nabla F(W^{(i)}_t) + \nabla F(W^{(i)}_t) - M^{(i)}_t\|_\F \nonumber \\
		& \leq \|M_t - \nabla F(W_t)\|_\F + \|\nabla F(W_t) -\nabla F(W^{(i)}_t)\|_\F + \|\nabla F(W^{(i)}_t) - M^{(i)}_t\|_\F \nonumber \\
		& \leq \|M_t - \nabla F(W_t)\|_\F + \|\nabla F(W^{(i)}_t) - M^{(i)}_t\|_\F  +
		L\|W_t-W^{(i)}_t\|_\F,
	\end{align}
	where the last inequality holds by Assumption~\ref{ass:s1}.
	
	When $j_1\neq i$ with probability $1-\frac{1}{N}$, since $W_0=W_0^{(i)}$, $M_1=\beta\nabla f(W_0;\xi_{j_1})$ and
	$M_1^{(i)}=\beta\nabla f(W_0^{(i)};\xi_{j_1})$, we have $M_1=M_1^{(i)}$.
	
	When $j_1= i$ with probability $\frac{1}{N}$, since $M_1=\beta\nabla f(W_0;\xi_{i})$ and
	$M_1^{(i)}=\beta\nabla f(W_0^{(i)};\tilde{\xi}_{1})$, we have
	\begin{align}
		& \E \|M_{1} - M^{(i)}_{1}\|_\F \nonumber \\
		& = \frac{1}{N}\|\beta\nabla f(W_0;\xi_{i})-\beta\nabla f(W_0^{(i)};\tilde{\xi}_{1})\|_\F  \nonumber \\
		& = \frac{\beta}{N} \|\nabla f(W_0;\xi_{i}) - \nabla F(W_0) + \nabla F(W_0) - \nabla F(W^{(i)}_0) +\nabla F(W^{(i)}_0) -\nabla f(W_0^{(i)};\tilde{\xi}_{1})\|_\F \nonumber \\
		& \leq  \frac{\beta}{N}\|\nabla f(W_0;\xi_{i}) - \nabla F(W_0)\|_\F + \frac{\beta}{N}\|\nabla F(W_0) - \nabla F(W^{(i)}_0)\|_\F \nonumber \\
		&\quad + \frac{\beta}{N}\|\nabla F(W^{(i)}_0) -\nabla f(W_0^{(i)};\tilde{\xi}_{1})\|_\F \nonumber \\
		& \leq \frac{2\beta\sigma}{N}.
	\end{align}
	Thus, we have $\E \|M_{1} - M^{(i)}_{1}\|_\F\leq \frac{2\beta\sigma}{N}$. According to the above
	inequality~(\ref{eq:w1}), we have
	\begin{align}
		\E\|W_{1} - W^{(i)}_{1}\|_\F
		\leq (1-\eta\lambda) \E\|W_{0} -W^{(i)}_{0}\|_\F + \frac{2\sqrt{2}\eta}{\kappa} \E\|M_1 - M^{(i)}_1\|_\F
		\leq \frac{4\sqrt{2}\beta\eta\sigma}{N\kappa},
	\end{align}
	where the last inequality is due to $W_{0} =W^{(i)}_{0}$ and $\E \|M_{1} - M^{(i)}_{1}\|_\F\leq \frac{2\beta\sigma}{N}$. 
	
	Let $\psi_1=\frac{2\beta\sigma}{N}$ and $\phi_1= \frac{4\sqrt{2}\beta\eta\sigma}{N\kappa}$. Further let $\eta=O(1)$, $\beta=O(1)$ and $\sigma=O(1)$,
	we have
	\begin{align}
		\E \|M_{1} - M^{(i)}_{1}\|_\F\leq \psi_1 = O(\frac{1}{N}), \quad
		\E \|W_{1} - W^{(i)}_{1}\|_\F\leq \phi_1 = O(\frac{1}{N\kappa}).
	\end{align}
	
	When $j_2\neq i$ with probability $1-\frac{1}{N}$, since $M_2=\beta\nabla f(W_1;\xi_{j_2})+ (1-\beta)M_{1}$ and
	$M^{(i)}_2=\beta\nabla f(W^{(i)}_1;\xi_{j_2})+ (1-\beta)M^{(i)}_{1}$, we have
	\begin{align} \label{eq:m2}
		& \E \|M_{2} - M^{(i)}_{2}\|_\F \nonumber \\
		& = (1-\frac{1}{N})\|\beta\nabla f(W_1;\xi_{j_2})+ (1-\beta)M_{1}-\beta\nabla f(W^{(i)}_1;\xi_{j_2})- (1-\beta)M^{(i)}_{1}\|_\F  \nonumber \\
		& \leq (1-\frac{1}{N})(1-\beta)\|M_{1}-M^{(i)}_{1}\|_\F + (1-\frac{1}{N})\beta\| \nabla f(W_1;\xi_{j_2})-\nabla f(W^{(i)}_1;\xi_{j_2})\|_\F \nonumber \\
		& \leq (1-\frac{1}{N})(1-\beta)\|M_{1}-M^{(i)}_{1}\|_\F + (1-\frac{1}{N})\beta L\|W_1 - W^{(i)}_1\|_\F \nonumber \\
		& \leq (1-\frac{1}{N})(1-\beta)\psi_1 + (1-\frac{1}{N})\beta L\phi_1.
	\end{align}
	When $j_2= i$ with probability $\frac{1}{N}$, since $M_2=\beta\nabla f(W_1;\xi_{i})+ (1-\beta)M_{1}$ and
	$M^{(i)}_2=\beta\nabla f(W^{(i)}_1;\tilde{\xi}_{i})+ (1-\beta)M^{(i)}_{1}$, we have
	\begin{align} \label{eq:m3}
		\E \|M_{2} - M^{(i)}_{2}\|_\F
		& = \frac{1}{N}\|\beta\nabla f(W_1;\xi_{i})+ (1-\beta)M_{1}-\beta\nabla f(W^{(i)}_1;\tilde{\xi}_{i})- (1-\beta)M^{(i)}_{1}\|_\F  \nonumber \\
		& \leq (1-\beta)\frac{1}{N}\|M_{1}-M^{(i)}_{1}\|_\F + \beta\frac{1}{N}\| \nabla f(W_1;\xi_{i})-\nabla f(W^{(i)}_1;\tilde{\xi}_{i})\|_\F \nonumber \\
		& = (1-\beta)\frac{1}{N}\|M_{1}-M^{(i)}_{1}\|_\F +  \beta\frac{1}{N}\big(\|\nabla f(W_1;\xi_{i}) - \nabla F(W_1)\|_\F \nonumber \\
		& \quad + \E \|\nabla F(W_1) - \nabla F(W_1^{(i)})\|_\F + \E\|\nabla F(W_1^{(i)})-\nabla f(W^{(i)}_1;\tilde{\xi}_{i})\|_\F \big) \nonumber \\
		& \leq (1-\beta)\frac{\psi_1}{N} +\frac{2\beta\sigma}{N} + \frac{\beta L\phi_1}{N}.
	\end{align}
	According to above inequalities~(\ref{eq:m2}) and~(\ref{eq:m3}), we have
	\begin{align} \label{eq:m4}
		\E \|M_{2} - M^{(i)}_{2}\|_\F
		& \leq (1-\frac{1}{N})(1-\beta)\psi_1 + (1-\frac{1}{N})\beta L\phi_1 + (1-\beta)\frac{\psi_1}{N} +\frac{2\beta\sigma}{N} + \frac{\beta L\phi_1}{N} \nonumber \\
		& = (1-\beta)\psi_1 + \beta L\phi_1 +\frac{2\beta\sigma}{N}=\psi_2,
	\end{align}
	where $\psi_2=(1-\beta)\psi_1 + \beta L\phi_1 +\frac{2\beta\sigma}{N}$.
	
	According to the above inequality~(\ref{eq:w1}), we have
	\begin{align}
		\|W_{2} - W^{(i)}_{2}\|_\F & \leq (1-\eta\lambda)\|W_{1} -W^{(i)}_{1}\|_\F + \frac{2\sqrt{2}\eta}{\kappa} \|M_2 - M^{(i)}_2\|_\F \nonumber \\
		& \leq (1-\eta\lambda)\phi_1 +\frac{2\sqrt{2}\eta}{\kappa} \psi_2 = \phi_2,
	\end{align}
	where $\phi_2 =(1-\eta\lambda)\phi_1 +\frac{2\sqrt{2}\eta}{\kappa} \psi_2$.
	Let $\eta=O(1)$, $\beta=O(1)$, $L=O(1)$ and $\sigma=O(1)$, since $\psi_1 = O(\frac{1}{N})$ and $\phi_1 = O(\frac{1}{N\kappa})$,
	we can obtain
	\begin{align}
		& \psi_2=(1-\beta)\psi_1 + \beta L\phi_1 +\frac{2\beta\sigma}{N} = O(\frac{1}{N\kappa})  \\
		& \phi_2 =(1-\eta\lambda)\phi_1 +\frac{2\sqrt{2}\eta}{\tau} \psi_2 = O(\frac{1}{N\kappa^2}).
	\end{align}
	
	Based on mathematical induction, we assume $\E \|M_t - M_t^{(i)}\| \leq \psi_t$ with
	$\psi_t=O(\frac{1}{N\kappa^{t-1}})$, and $\E\|W_t-W_t^{(i)}\|\leq \phi_t$ with $\phi_t=O(\frac{1}{N\kappa^t})$.
	
	When $j_{t+1}\neq i$ with probability $1-\frac{1}{N}$,
	since $M_{t+1}=\beta\nabla f(W_t;\xi_{j_{t+1}})+ (1-\beta)M_{t}$ and
	$M^{(i)}_{t+1}=\beta\nabla f(W^{(i)}_t;\xi_{j_{t+1}})+ (1-\beta)M^{(i)}_{t}$, we have
	\begin{align} \label{eq:m5}
		& \E \|M_{t+1} - M^{(i)}_{t+1}\|_\F \nonumber \\
		& = (1-\frac{1}{N})\|\beta\nabla f(W_t;\xi_{j_{t+1}})+ (1-\beta)M_{t}-\beta\nabla f(W^{(i)}_t;\xi_{j_{t+1}})- (1-\beta)M^{(i)}_{t}\|_\F  \nonumber \\
		& \leq (1-\frac{1}{N})(1-\beta)\|M_{t}-M^{(i)}_{t}\|_\F + (1-\frac{1}{N})\beta\| \nabla f(W_t;\xi_{j_{t+1}})-\nabla f(W^{(i)}_t;\xi_{j_{t+1}})\|_\F \nonumber \\
		& \leq (1-\frac{1}{N})(1-\beta)\|M_{t}-M^{(i)}_{t}\|_\F + (1-\frac{1}{N})\beta L\|W_t - W^{(i)}_t\|_\F \nonumber \\
		& \leq (1-\frac{1}{N})(1-\beta)\psi_t + (1-\frac{1}{N})\beta L\phi_t.
	\end{align}
	When $j_{t+1}= i$ with probability $\frac{1}{N}$, since $M_{t+1}=\beta\nabla f(W_t;\xi_{i})+ (1-\beta)M_{t}$ and
	$M^{(i)}_{t+1}=\beta\nabla f(W^{(i)}_t;\tilde{\xi}_{i})+ (1-\beta)M^{(i)}_{t}$, we have
	\begin{align} \label{eq:m6}
		\E \|M_{t+1} - M^{(i)}_{t+1}\|_\F
		& = \frac{1}{N}\|\beta\nabla f(W_t;\xi_{i})+ (1-\beta)M_{t}-\beta\nabla f(W^{(i)}_t;\tilde{\xi}_{i})- (1-\beta)M^{(i)}_{t}\|_\F  \nonumber \\
		& \leq (1-\beta)\frac{1}{N}\|M_{t}-M^{(i)}_{t}\|_\F + \beta\frac{1}{N}\| \nabla f(W_t;\xi_{i})-\nabla f(W^{(i)}_t;\tilde{\xi}_{i})\|_\F \nonumber \\
		& = (1-\beta)\frac{1}{N}\|M_{t}-M^{(i)}_{t}\|_\F +  \beta\frac{1}{N}\big(\|\nabla f(W_t;\xi_{i}) - \nabla F(W_t)\|_\F \nonumber \\
		& \quad + \E \|\nabla F(W_t) - \nabla F(W_t^{(i)})\|_\F + \E\|\nabla F(W_t^{(i)})-\nabla f(W_t^{(i)};\tilde{\xi}_{i})\|_\F \big) \nonumber \\
		& \leq (1-\beta)\frac{\psi_t}{N} +\frac{2\beta\sigma}{N} + \frac{\beta L\phi_t}{N}.
	\end{align}
	Then we have
	\begin{align} \label{eq:m7}
		\E \|M_{t+1} - M^{(i)}_{t+1}\|_\F
		& \leq (1-\frac{1}{N})(1-\beta)\psi_t + (1-\frac{1}{N})\beta L\phi_t + (1-\beta)\frac{\psi_t}{N} +\frac{2\beta\sigma}{N} + \frac{\beta L\phi_t}{N} \nonumber \\
		& \leq (1-\beta)\psi_t + \beta L\phi_t+\frac{2\beta\sigma}{N} = \psi_{t+1},
	\end{align}
	where $\psi_{t+1} = (1-\beta)\psi_t + \beta L\phi_t+\frac{2\beta\sigma}{N}$.
	
	According to the above inequalities~(\ref{eq:w1}) and~(\ref{eq:m7}), we have
	\begin{align}
		\|W_{t+1} - W^{(i)}_{t+1}\|_\F & \leq (1-\eta\lambda)\|W_{t} -W^{(i)}_{t}\|_\F + \frac{2\sqrt{2}\eta}{\kappa} \|M_{t+1} - M^{(i)}_{t+1}\|_\F \nonumber \\
		& \leq (1-\eta\lambda)\phi_{t} +\frac{2\sqrt{2}\eta}{\kappa} \psi_{t+1} = \phi_{t+1},
	\end{align}
	where $\phi_{t+1} = (1-\eta\lambda)\phi_{t} +\frac{2\sqrt{2}\eta}{\kappa} \psi_{t+1}$.
	Let $\eta=O(1)$, $\beta=O(1)$, $L=O(1)$ and $\sigma=O(1)$, since $\psi_t=O(\frac{1}{N\kappa^{t-1}})$ and $\phi_t=O(\frac{1}{N\kappa^t})$,
	we can obtain
	\begin{align}
		& \psi_{t+1} = (1-\beta)\psi_t + \beta L\phi_t+\frac{2\beta\sigma}{N} = O(\frac{1}{N\kappa^{t}})  \\
		& \phi_{t+1} = (1-\eta\lambda)\phi_{t} +\frac{2\sqrt{2}\eta}{\kappa} \psi_{t+1} = O(\frac{1}{N\kappa^{t+1}}).
	\end{align}
	By using mathematical induction, we have
	\begin{align} \label{eq:w2}
		\E \|W_{T} - W_{T}^{(i)}\| \leq O\Big(\frac{1}{N\kappa^{T}}\Big).
	\end{align}
	
	By using Assumption~\ref{ass:g}, i,e., $G$-Lipschitz $f(W;\xi)$ for any $\xi \sim \mathcal{D}$, we have
	\begin{align}
		\E |f(W_T;\xi)-f(W^{(i)}_T;\xi)| \leq G \E \|W_{T} - W^{(i)}_{T}\|_\F \leq O\Big(\frac{1}{N\kappa^{T}}\Big),
	\end{align}
	where the last inequality is due to the above inequality~(\ref{eq:w2}) and $G=O(1)$.
	
	By using the lemma~\ref{lem:gs}, i.e., the uniform stability bound~\cite{shalev2010learnability,hardt2016train}, and taking expectations over $S$, $S^{(i)}$ and the algorithm's randomness, we
	can obtain
	\begin{align}
		|\E [F(W_T) - F_S(W_T)]| \leq O\Big(\frac{1}{N\kappa^{T}}\Big).
	\end{align}
	
\end{proof}

\begin{theorem} \label{th:t2} (\textbf{Restatement of Theorem~\ref{th:2}})
	Assume the sequence $\{W_t,M_t\}_{t=0}^T$ is generated
	from Algorithm~\ref{alg:2} on dataset $S=\{\xi_1,\xi_2,\cdots,\xi_N\}$. Under the Assumptions~\ref{ass:s1},~\ref{ass:g},~\ref{ass:v}, without loss of generality, let $\tau\geq 1$, and let
	$\eta=O(1)$, $\beta=O(1)$ with $\beta\in (0,1]$ and $\lambda=0(1)$ with $\lambda <\frac{1}{\eta}$. When the iteration number is relatively small (i.e., $T=O(1)$) set
	$\eta=O(1)$, otherwise set $\eta=\frac{1}{T}$, then we have
	\begin{align}
		|\E [F(W_T) - F_S(W_T)] | \leq O\Big(\frac{1}{N}\Big). \nonumber
	\end{align}
\end{theorem}

\begin{proof}
	Implementing Algorithm~\ref{alg:2} on datasets $S$ and $S^{(i)}$ with the same random index sequence $\{j_t\}_{t=1}^T$, and let $\{W_t\}_{t=1}^T$ and $\{W_t^{(i)}\}_{t=1}^T$ be generated from Algorithm~\ref{alg:2} with $S$ and $S^{(i)}$.
	Let $M_t = \beta\nabla f(W_{t-1};\xi_{t}) + (1-\beta)M_{t-1}$ and $M^{(i)}_t = \beta\nabla f(W_{t-1}^{(i)};\xi_{t}) + (1-\beta)M^{(i)}_{t-1}$ is generated from Algorithm~\ref{alg:2}.
	
	Without loss of generality, let $\tau \geq 1$ in Algorithm~\ref{alg:2}. When $\min_{i\neq j\in S_t} |\Sigma_{t,ii}-\Sigma_{t,jj}| \geq \tau \geq 1$ with $S_t=\{i|\Sigma_{t,ii} \neq 0, 1\leq i\leq d\}$,
	since $W_{t} = W_{t-1} - \eta (U_tV_t^\top+\lambda W_{t-1})$ and  $W^{(i)}_{t} = W^{(i)}_{t-1} - \eta (U^{(i)}_t(V^{(i)}_t)^\top+\lambda W^{(i)}_{t-1})$, we have
	\begin{align} \label{eq:w3}
		& \|W_{t} - W^{(i)}_{t}\|_\F \nonumber \\
		& = \|(1-\eta\lambda)(W_{t-1} -W^{(i)}_{t-1}) - \eta (U_tV_t^\top - U^{(i)}_t(V^{(i)}_t)^\top)\|_\F \nonumber \\
		& \leq (1-\eta\lambda)\|W_{t-1} -W^{(i)}_{t-1}\|_\F + \eta \|U_tV_t^\top - U^{(i)}_t(V^{(i)}_t)^\top\|_\F \nonumber \\
		& \leq (1-\eta\lambda)\|W_{t-1} -W^{(i)}_{t-1}\|_\F + \eta \|U_tV_t^\top -U_t(V^{(i)}_t)^\top \|_\F + \eta\|U_t(V^{(i)}_t)^\top - U^{(i)}_t(V^{(i)}_t)^\top\|_\F \nonumber \\
		& \leq (1-\eta\lambda)\|W_{t-1} -W^{(i)}_{t-1}\|_\F + \eta \|V_t -V^{(i)}_t \|_\F\|U_t\|_2 + \eta\|U_t - U^{(i)}_t\|_\F\|V^{(i)}_t\|_2 \nonumber \\
		& = (1-\eta\lambda)\|W_{t-1} -W^{(i)}_{t-1}\|_\F + \eta \|V_t -V^{(i)}_t \|_\F + \eta\|U_t - U^{(i)}_t\|_\F \nonumber \\
		& \mathop{\leq}^{(i)} (1-\eta\lambda)\|W_{t-1} -W^{(i)}_{t-1}\|_\F + \frac{2\sqrt{2}\eta}{\kappa_t} \|M_t - M^{(i)}_t\|_\F  \nonumber \\
		& \mathop{\leq}^{(ii)} (1-\eta\lambda)\|W_{t-1} -W^{(i)}_{t-1}\|_\F + \frac{2\sqrt{2}\eta}{\tau} \|M_t - M^{(i)}_t\|_\F \nonumber \\
		& \leq (1-\eta\lambda)\|W_{t-1} -W^{(i)}_{t-1}\|_\F + 2\sqrt{2}\eta\|M_t - M^{(i)}_t\|_\F,
	\end{align}
	where the inequality~$(i)$ holds by the singular-vector version of the Davis-Kahan $\sin\theta$ theorem (Corollary 3 in \cite{yu2015useful}) with $\kappa_t = \min_{i\neq j}|\sigma_i(M_t)-\sigma_j(M_t)|>0$, and
	the inequality~$(ii)$ is due to $\kappa_t \geq \tau\geq 1$, and the last inequality holds by $\tau\geq 1$.
	
	When $\min_{i\neq j\in S_t} |\Sigma_{t,ii}-\Sigma_{t,jj}| < \tau $
	with $S_t=\{i|\Sigma_{t,ii} \neq 0, 1\leq i\leq d\}$,
	since $W_{t} = W_{t-1} - \eta (M_t+\lambda W_{t-1})$ and  $W^{(i)}_{t} = W^{(i)}_{t-1} - \eta (M^{(i)}_t+\lambda W^{(i)}_{t-1})$, we have
	\begin{align} \label{eq:w4}
		\|W_{t} - W^{(i)}_{t}\|_\F & = \|(1-\eta\lambda)(W_{t-1} -W^{(i)}_{t-1}) - \eta (M_t - M^{(i)}_t)\|_\F \nonumber \\
		& \leq (1-\eta\lambda)\|W_{t-1} -W^{(i)}_{t-1}\|_\F + \eta \|M_t - M^{(i)}_t\|_\F \nonumber \\
		& \leq (1-\eta\lambda)\|W_{t-1} -W^{(i)}_{t-1}\|_\F + 2\sqrt{2}\eta\|M_t - M^{(i)}_t\|_\F,
	\end{align}
	Thus, we have
	\begin{align} \label{eq:w5}
		\|W_{t} - W^{(i)}_{t}\|_\F \leq (1-\eta\lambda)\|W_{t-1} -W^{(i)}_{t-1}\|_\F + 2\sqrt{2}\eta\|M_t - M^{(i)}_t\|_\F.
	\end{align}
	According to the above inequality~(\ref{eq:w5}), following the above proof of Theorem~\ref{th:1},
	we have
	\begin{align}
		\E \|M_{1} - M^{(i)}_{1}\|_\F\leq \frac{2\beta\sigma}{N}=\psi_1, \quad
		\E\|W_{1} - W^{(i)}_{1}\|_\F \leq \frac{4\sqrt{2}\beta\eta\sigma}{N}=\phi_1.
	\end{align}
	Let $\eta=O(1)$, $\beta=O(1)$ and $\sigma=O(1)$, we can obtain
	\begin{align}
		\E \|M_{1} - M^{(i)}_{1}\|_\F\leq \psi_1 =O(\frac{1}{N}), \quad
		\E\|W_{1} - W^{(i)}_{1}\|_\F \leq \phi_1=O(\frac{1}{N}).
	\end{align}
	
	\textbf{If} the iteration number is small (i.e., $T=O(1)$),
	let $\eta=O(1)$, $\lambda\in [0,\frac{1}{\eta})$, $\beta=O(1)$ with $\beta\in (0,1)$, $\sigma=O(1)$ and $L=O(1)$.
	Following the above proof of Theorem~\ref{th:1}, assume $\E \|M_t - M_t^{(i)}\|_\F \leq \psi_t$ with
	$\psi_t=O(\frac{1}{N})$, and $\E\|W_t-W^{(i)}_t\|_\F \leq \phi_t$ with $\phi_t=O(\frac{1}{N})$, then we have
	\begin{align}
		& \E \|M_{t+1} - M^{(i)}_{t+1}\|_\F \leq \psi_{t+1}=(1-\beta)\psi_t + \beta L\phi_t+\frac{2\beta\sigma}{N}=O(\frac{1}{N}), \nonumber \\
		& \E \|W_{t+1} - W^{(i)}_{t+1}\|_\F \leq \phi_{t+1}=(1-\eta\lambda)\phi_{t} +2\sqrt{2}\eta \psi_{t+1} =O(\frac{1}{N}). \nonumber
	\end{align}
	By using the mathematical induction, then we can obtain
	\begin{align} \label{eq:w6}
		\E \|W_{T} - W^{(i)}_{T}\|_\F  \leq  O(\frac{1}{N}).
	\end{align}
	
	\textbf{If} the iteration number $T$ is large, we consider the iteration number $t\geq1$
	in the generalization analysis. By using the mathematical induction, due to recursion of the above inequality~(\ref{eq:w5}), following the above proof of Theorem~\ref{th:1}, we assume $\E \|M_t - M^{(i)}_t\| \leq \psi_t$ with $\psi_t=O(\frac{t}{N})$, and $\E\|W_t-W_t^{(i)}\|\leq \phi_t$ with $\phi_t=O(\frac{t}{N})$.
	
	Let $\eta=O(1)$, $\lambda\in [0,\frac{1}{\eta})$, $\beta=O(1)$ with $\beta\in (0,1)$, $\sigma=O(1)$ and
	$L=O(1)$, we have
	\begin{align}
		& \E \|M_{t+1} - M^{(i)}_{t+1}\|_\F \leq \psi_{t+1} = (1-\beta)\psi_t + \beta L\phi_t+\frac{2\beta\sigma}{N} = O(\frac{t+1}{N})  \\
		& \E \|W_{t+1} - W^{(i)}_{t+1}\|_\F \leq \phi_{t+1} = (1-\eta\lambda)\phi_{t} +2\sqrt{2}\eta \psi_{t+1} = O(\frac{t+1}{N})
	\end{align}
	By using the mathematical induction, we have
	\begin{align}
		\E \|W_{T} - W_{T}^{(i)}\| \leq  O(\frac{T}{N}).
	\end{align}
	
	We further let $\eta=O( \textcolor{blue}{\frac{1}{T}})$, $\lambda\in [0,\frac{1}{\eta})$, $\beta=O(1)$ with $\beta\in (0,1)$, $\sigma=O(1)$ and
	$L=O(1)$, we can obtain
	\begin{align} \label{eq:w7}
		\E \|W_{T} - W^{(i)}_{T}\|  \leq  O(\frac{1}{N}).
	\end{align}
	
	By using Assumption~\ref{ass:g}, i,e., $G$-Lipschitz $f(W;\xi)$ for any $\xi \sim \mathcal{D}$, we have
	\begin{align}
		\E |f(W_T;\xi)-f(W^{(i)}_T;\xi)| \leq G \E \|W_{T} - W^{(i)}_{T}\|_\F \leq O\Big(\frac{1}{N}\Big),
	\end{align}
	where the last inequality is due to the above inequality~(\ref{eq:w6}) or (\ref{eq:w7}) and $G=O(1)$.
	
	By using the above lemma~\ref{lem:gs}, i.e., the uniform stability bound~\cite{shalev2010learnability,hardt2016train}, and taking expectations over $S$, $S^{(i)}$ and the algorithm's randomness, we
	can obtain
	\begin{align}
		|\E[F(W_T) - F_S(W_T)]| \leq O\Big(\frac{1}{N}\Big).
	\end{align}

\end{proof}

\section{ Convergence Analysis}
In this section, we provide the detailed convergence analysis for our MiMuon optimizer
under some mild conditions.

\begin{lemma} \label{lem:mb2} (\textbf{Restatement of Lemma~\ref{lem:mb}})
	Assume the sequence $\{W_t,M_t\}_{t=0}^T$ is generated from Algorithm~\ref{alg:2}.
	Under the Assumptions~\ref{ass:s1},~\ref{ass:v}, given
	$\|W_0\|_\F \leq \eta\hat{G} $ and $\lambda \leq \frac{1}{2(1+T)\eta\hat{G}}$, we have
	\begin{align}
		& \|W_{t}-W_{t-1}\|_\F \leq \eta\breve{G}, \quad \|W_{t}\|_\F \leq (t+1)\eta\hat{G}, \ \forall t\geq 1\nonumber \\
		& \frac{1}{T+1}\sum_{t=0}^{T} \mathbb{E} \|\nabla F(W_t) - M_{t+1}\|_\F  \leq \frac{\sigma}{\sqrt{T\beta}} +  \frac{\sqrt{2}}{\beta}L\eta\breve{G}  + \sqrt{\beta}\sigma,
	\end{align}
	where $\hat{G}=\max(G,\sqrt{r})$ and $\breve{G}=\hat{G}+\frac{1}{2}$.
\end{lemma}

\begin{proof}
	Let $\hat{G}=\max(G,\sqrt{r})$.
	When $\min_{i\neq j\in S_t} |\Sigma_{t,ii}-\Sigma_{t,jj}| \geq \tau$ with $S_t=\{i|\Sigma_{t,ii} \neq 0, 1\leq i\leq d\}$, since $W_{t} = W_{t-1} - \eta (U_{t}V_{t}^\top + \lambda W_{t-1})$ from the line 8 of Algorithm~\ref{alg:2}, we have
	\begin{align} \label{eq:g1}
		\|W_t\|_\F & = \| W_{t-1} - \eta (U_{t}V_{t}^\top + \lambda W_{t-1})\|_\F \nonumber \\
		& = \| (1-\eta\lambda)W_{t-1} - \eta U_{t}V_{t}^\top \|_\F \nonumber \\
		& \leq (1-\eta\lambda) \|W_{t-1}\|_\F + \eta \| U_{t}V_{t}^\top\|_\F \nonumber \\
		& \leq (1-\eta\lambda) \|W_{t-1}\|_\F + \eta \sqrt{r} \nonumber \\
		& \leq (1-\eta\lambda)^{t}\|W_0\|_\F + t\eta\sqrt{r} \nonumber \\
		& \leq (t+1)\eta\sqrt{r},
	\end{align}
	where the first inequality is due to $\lambda \leq \frac{1}{2(1+T)\eta\hat{G}} <\frac{1}{\eta}$, and the last inequality holds by $\|W_0\|_\F \leq \eta\hat{G} \leq \eta \sqrt{r}$.
	Since $W_{t} = W_{t-1} - \eta (U_{t}V_{t}^\top + \lambda W_{t-1})$,
	then we have
	\begin{align} \label{eq:g2}
		\|W_t-W_{t-1}\|_\F & = \|\eta (U_{t}V_{t}^\top + \lambda W_{t-1})\|_F \nonumber \\
		& \leq \eta \|U_{t}V_{t}^\top\|_\F + \eta\lambda\| W_{t-1}\|_\F \nonumber \\
		& \leq \eta\sqrt{r} + \eta\lambda t \eta\sqrt{r} \nonumber \\
		& \leq \eta\sqrt{r} + \eta\frac{1}{2(1+T)\eta\sqrt{r}}t \eta\sqrt{r} \nonumber \\
		& \leq \eta(\sqrt{r}+\frac{1}{2}),
	\end{align}
	where the second last inequality is due to
	$\lambda \leq \frac{1}{2(1+T)\eta\hat{G}}\leq \frac{1}{2(1+T)\eta\sqrt{r}}$.
	
	Since $M_t$ is exponential moving average of stochastic gradients, by using Assumption~\ref{ass:g}, we have $\|M_t\|\leq G$.
	When $\min_{i\neq j\in S_t} |\Sigma_{t,ii}-\Sigma_{t,jj}| < \tau$, according to
	the line 10 of Algorithm~\ref{alg:2}, we have
	\begin{align} \label{eq:g3}
		\|W_t\|_\F & = \|W_{t-1}- \eta (M_t+\lambda W_{t-1})\|_\F \nonumber \\
		& = \|(1-\eta\lambda)W_{t-1} + \eta M_t\|_\F \nonumber \\
		& \leq (1-\eta\lambda)\|W_{t-1}\|_\F + \eta \|M_t\|_\F \nonumber \\
		& \leq (1-\eta\lambda)\|W_{t-1}\|_\F + \eta G \nonumber \\
		& \leq (1-\eta\lambda)^t\|W_{0}\|_\F + t\eta G \nonumber \\
		& \leq (t+1)\eta G,
	\end{align}
	where the last inequality is due to $\|W_{0}\|_\F \leq \eta\hat{G} \leq \eta G$.
	Since $W_{t} = W_{t-1} - \eta (M_t + \lambda W_{t-1})$,
	then we have
	\begin{align} \label{eq:g4}
		\|W_t-W_{t-1}\|_\F & = \|\eta (M_t + \lambda W_{t-1})\|_F \nonumber \\
		& \leq \eta \|M_t\|_\F + \eta\lambda\|W_{t-1}\|_\F \nonumber \\
		& \leq \eta G + \eta\lambda t \eta G\nonumber \\
		& \leq \eta G + \eta\frac{1}{2(1+T)\eta \hat{G}}t \eta G \nonumber \\
		& \leq \eta(G+\frac{1}{2}),
	\end{align}
	where the second last inequality is due to
	$\lambda \leq \frac{1}{2(1+T)\eta\hat{G}}\leq \frac{1}{2(1+T)\eta G}$.
	
	According to the above inequalities~(\ref{eq:g1}), (\ref{eq:g2}), (\ref{eq:g3})
	and (\ref{eq:g4}),
	set $\hat{G}=\max(G,\sqrt{r})$ and $\lambda \leq \frac{1}{2(1+T)\eta\hat{G}}$,
	we have $\|W_t\|_\F\leq (t+1)\eta \hat{G}$ and
	$\|W_t-W_{t-1}\|_\F \leq \eta(\hat{G}+\frac{1}{2})$ for all $t\geq1$.
	Further let $\breve{G}=\hat{G}+\frac{1}{2}$,
	we have $\|W_t-W_{t-1}\|_\F \leq \eta\breve{G}$ for all $t\geq1$.
	
	Let $\Theta_t = \nabla F(W_t) - M_{t+1}$.
	Since $M_{t+1}=\beta\nabla f(W_{t};\xi_{t+1}) + (1-\beta)M_{t}$, we have
	\begin{align}
		& \mathbb{E} \|\Theta_t\|_\F^2  \nonumber \\
		& = \mathbb{E}\big[ \|\nabla F(W_t) - M_{t+1}\|_\F^2\big] \nonumber \\
		& = \mathbb{E}\big[\|(1-\beta)(\nabla F(W_{t-1}) - M_{t}) + \beta(\nabla F(W_t)- \nabla f(W_t;\xi_{t+1})) + (1-\beta)\big( \nabla F(W_t)-\nabla F(W_{t-1})\big)\|_\F^2\big] \nonumber \\
		& \mathop{=}^{(i)} (1-\beta)^2\mathbb{E} \|\nabla F(W_{t-1}) - M_{t} +(\nabla F(W_t)-\nabla F(W_{t-1}))\|_\F^2  + \beta^2\mathbb{E} \|\nabla F(W_t)-\nabla f(W_t;\xi_{t+1}) \|_\F^2 \nonumber \\
		& \mathop{\leq}^{(ii)} (1 - \beta)^2(1 + \beta)\mathbb{E} \|\nabla F(W_{t-1})  -  M_{t}\|_\F^2  +  (1 - \beta)^2(1+\frac{1}{\beta})\mathbb{E}\|\nabla F(W_t) - \nabla F(W_{t-1})\|_\F^2   + \beta^2\sigma^2 \nonumber \\
		& \mathop{\leq}^{(iii)} (1-\beta)\mathbb{E} \|\Theta_{t-1}\|_\F^2 + \frac{2}{\beta}\mathbb{E}\|\nabla F(W_t)-\nabla F(W_{t-1})\|_\F^2  + \beta^2\sigma^2 \nonumber \\
		& \mathop{\leq}^{(iv)}  (1-\beta)\mathbb{E} \|\Theta_{t-1}\|_\F^2 + \frac{2}{\beta}L^2\mathbb{E}\| W_t - W_{t-1}\|_\F^2  + \beta^2\sigma^2 \nonumber \\
		& \leq (1-\beta)\mathbb{E} \|\Theta_{t-1}\|_\F^2 + \frac{2}{\beta}L^2\eta^2\breve{G}^2  + \beta^2\sigma^2 ,
	\end{align}
	where the equality $(i)$ holds by $\mathbb{E}[\nabla f(W_t;\xi_{t+1})]=\nabla F(W_t)$, and
	the inequality $(ii)$ holds by Young's inequality and Assumption~\ref{ass:v}, and the inequality $(iii)$ is due to $0< \beta  \leq 1$ such that  $(1-\beta )^2(1+\beta )=1-\beta -\beta ^2+
	\beta ^3\leq 1-\beta $ and $(1-\beta )^2(1+\frac{1}{\beta }) \leq 1+\frac{1}{\beta } \leq \frac{2}{\beta }$, and the inequality $(iv)$ holds by Assumption~\ref{ass:s1}, and the last inequality holds by
	$\|W_t-W_{t-1}\|_\F \leq \eta\breve{G}$.
	
	By expanding the recursion, then we have
	\begin{align} \label{eq:mb}
		\mathbb{E} \|\Theta_t\|_\F^2 & \leq (1-\beta)^t\|\Theta_{0}\|_\F^2 + ( \frac{2}{\beta}L^2\eta^2\breve{G}^2  + \beta^2\sigma^2)\sum_{s=1}^t(1-\beta)^{t-s} \nonumber \\
		& \leq (1-\beta)^t\sigma^2 +  \frac{2}{\beta^2}L^2\eta^2\breve{G}^2  + \beta\sigma^2.
	\end{align}
	
	By summing the above inequality~(\ref{eq:mb}) from $t=0$ to $T$, we can obtain
	\begin{align}
		\frac{1}{T+1}\sum_{t=0}^{T} \mathbb{E} \|\Theta_t\|_\F^2 & \leq \frac{1}{T+1}\sum_{t=0}^{T}(1-\beta)^t\sigma^2 +  \frac{2}{\beta^2}L^2\eta^2\breve{G}^2  + \beta\sigma^2 \nonumber \\
		& \leq \frac{\sigma^2}{T\beta} +  \frac{2}{\beta^2}L^2\eta^2\breve{G}^2  + \beta\sigma^2.
	\end{align}
	
	By using Jensen inequality, then we have
	\begin{align}
		\frac{1}{T+1}\sum_{t=0}^{T} \mathbb{E} \|\Theta_t\|_\F & \leq \sqrt{\frac{1}{T+1}\sum_{t=0}^T \mathbb{E} \|\Theta_t\|_\F^2} \leq \sqrt{\frac{\sigma^2}{T\beta} +  \frac{2}{\beta^2}L^2\eta^2\breve{G}^2  + \beta\sigma^2}  \nonumber \\
		&\leq \frac{\sigma}{\sqrt{T\beta}} +  \frac{\sqrt{2}}{\beta}L\eta\breve{G}  + \sqrt{\beta}\sigma,
	\end{align}
	where the last inequality is due to $\sqrt{a+b+c}\leq \sqrt{a} + \sqrt{b} +\sqrt{c}$ with $a,b,c\geq0$.
	
\end{proof}

\begin{theorem} (\textbf{Restatement of Theorem~\ref{th:3}})
	Assume the sequence $\{W_t\}_{t=0}^T$ is generated from Algorithm~\ref{alg:2}. Under the Assumptions~\ref{ass:s1},~\ref{ass:g}~\ref{ass:v}~\ref{ass:b},
	and let $\|W_0\|_\F \leq \eta\hat{G} $ and $\lambda \leq \frac{1}{2(1+T)\eta\hat{G}}$, we have
	\begin{align}
		\frac{1}{T+1} \sum_{t=0}^{T} \E \|\nabla F(W_t)\|_\F & \leq \frac{2(F(W_0)-F^*)}{T\eta}  + 4\sqrt{r}\Big( \frac{\sigma}{\sqrt{T\beta}} +  \frac{\sqrt{2}}{\beta}L\eta\breve{G}  + \sqrt{\beta}\sigma \Big) + L\eta\breve{G}^2.
	\end{align}
\end{theorem}
\begin{proof}
	When $\min_{i\neq j\in S_t} |\Sigma_{t,ii}-\Sigma_{t,jj}| \geq \tau$, according to the line 8 of Algorithm~\ref{alg:2}, we have $W_{t} = W_{t-1} - \eta (U_{t}V_{t}^\top+\lambda W_{t-1})$ for all $t\geq1$.
	By using Assumption~\ref{ass:s1}, then we have
	\begin{align} \label{eq:c1}
		F(W_{t}) & \leq F(W_{t-1})+ \langle\nabla F(W_{t-1}), W_{t}-W_{t-1} \rangle+\frac{L}{2}\|W_{t}-W_{t-1} \|_\F^2 \nonumber \\
		\leq & F(W_{t-1})- \eta\langle\nabla F(W_{t-1}), U_tV_t^\top + \lambda W_{t-1}\rangle + \frac{L}{2}\eta^2\breve{G}^2 \nonumber \\
		\leq & F(W_{t-1}) - \eta\langle M_t, U_tV_t^\top \rangle - \eta\lambda\langle \nabla F(W_{t-1}), W_{t-1}\rangle -\eta\langle \nabla F(W_{t-1})-M_t , U_tV_t^\top\rangle + \frac{L}{2}\eta^2\breve{G}^2   \nonumber \\
		\leq& F(W_{t-1}) - \eta\langle M_t, U_tV_t^\top \rangle + \eta\lambda\|\nabla F(W_{t-1})\|_\F \|W_{t-1}\|_\F -\eta\langle \nabla F(W_{t-1})-M_t , U_tV_t^\top\rangle + \frac{L}{2}\eta^2\breve{G}^2    \nonumber  \\
		\mathop{\leq}^{(i)} & F(W_{t-1})- \eta\|M_t\|_* + \eta\lambda t\eta \hat{G}\|\nabla F(W_{t-1})\|_\F +\eta \sqrt{r}\|\nabla F(W_{t-1})-M_t\|_\F \nonumber \\
		&\quad + \frac{L}{2}\eta^2\breve{G}^2  \nonumber \\
		\mathop{\leq}^{(ii)} & F(W_{t-1})- \eta\|M_t\|_* + \frac{\eta}{2} \|\nabla F(W_{t-1})\|_\F  + \eta\sqrt{r}\|\nabla F(W_{t-1})-M_t\|_\F + \frac{L}{2}\eta^2\breve{G}^2  \nonumber \\
		\mathop{\leq}^{(iii)} & F(W_{t-1})- \eta\|\nabla F(W_{t-1})\|_* + \frac{\eta}{2} \|\nabla F(W_{t-1})\|_\F  + 2\eta\sqrt{r}\|\nabla F(W_{t-1})-M_t\|_\F + \frac{L}{2}\eta^2\breve{G}^2  \nonumber \\
		\leq & F(W_{t-1})- \frac{\eta}{2}\|\nabla F(W_{t-1})\|_*  + 2\eta\sqrt{r}\|\nabla F(W_{t-1})-M_t\|_\F + \frac{L}{2}\eta^2\breve{G}^2,
	\end{align}
	where the first inequality is due to $\|W_{t}-W_{t-1}\|_\F \leq \eta\breve{G}$, and the above inequality $(i)$ holds by $\|W_{t-1}\|_\F \leq t\eta\hat{G}$ and $\langle M_t , U_tV_t^\top \rangle =\|M_t\|_*$, and
	the above inequality $(ii)$ is due to $\lambda \leq \frac{1}{2(1+T)\eta\hat{G}}$,
	and the above inequality $(iii)$
	holds by $\|M_t\|_*\geq \|\nabla F(W_{t-1})\|_* - \|\nabla F(W_{t-1})-M_t\|_*\geq \|\nabla F(W_{t-1})\|_* - \sqrt{r}\|\nabla F(W_{t-1})-M_t\|_\F$ and $\|M_t\|_\F\leq \|M_t\|_*\leq \sqrt{r}\|M_t\|_\F$.
	
	Then we rewrite the above inequality~(\ref{eq:c1}), and take expectation on this inequality as follows:
	\begin{align} \label{eq:c2}
		\E \|\nabla F(W_{t-1})\|_\F \leq \E \|\nabla F(W_{t-1})\|_* & \leq \E \big[ \frac{2(F(W_{t-1})-F(W_{t}))}{\eta} \big] \nonumber \\
		& \quad + 4\sqrt{r} \E \|\nabla F(W_{t-1})-M_t\|_\F + L\eta\breve{G}^2.
	\end{align}
	
	When $\min_{i\neq j\in S_t} |\Sigma_{t,ii}-\Sigma_{t,jj}| < \tau$, according to the line 10 of Algorithm~\ref{alg:2}, we have $W_{t} = W_{t-1} - \eta (M_t + \lambda W_{t-1})$.
	By using Assumption~\ref{ass:s1}, then we have
	\begin{align} \label{eq:c3}
		F(W_{t}) \leq & F(W_{t-1})+ \langle\nabla F(W_{t-1}), W_{t}-W_{t-1} \rangle+\frac{L}{2}\|W_{t}-W_{t-1}\|_\F^2 \nonumber \\
		\leq & F(W_{t-1})- \eta\langle\nabla F(W_{t-1}), M_t + \lambda W_{t-1}\rangle + \frac{L}{2}\eta^2\breve{G}^2 \nonumber \\
		\leq & F(W_{t-1}) - \eta\langle\nabla F(W_{t-1}), M_t \rangle - \eta\lambda\langle \nabla F(W_{t-1}), W_{t-1}\rangle + \frac{L}{2}\eta^2\breve{G}^2  \nonumber \\
		\mathop{\leq}^{(i)} & F(W_{t-1}) - \frac{\eta}{2}\|\nabla F(W_{t-1})\|_\F- \frac{\eta}{2}\|M_t\|_\F + \frac{\eta}{2}\|\nabla F(W_{t-1})-M_t\|_\F \nonumber \\
		&\quad + \eta\lambda\|\nabla F(W_{t-1})\|_\F \|W_{t-1}\|_\F + \frac{L}{2}\eta^2\breve{G}^2  \nonumber \\
		\mathop{\leq}^{(ii)} & F(W_{t-1}) - \frac{\eta}{2}\|\nabla F(W_{t-1})\|_\F - \frac{\eta}{2}\|M_t\|_\F+ \frac{\eta}{2}\|\nabla F(W_{t-1})-M_t\|_\F \nonumber \\
		&\quad + \eta\lambda t\eta \hat{G}\|\nabla F(W_{t-1})\|_\F+ \frac{L}{2}\eta^2\breve{G}^2  \nonumber \\
		\mathop{\leq}^{(iii)} & F(W_{t-1})  - \frac{\eta}{2}\|M_t\|_\F+ \frac{\eta}{2}\|\nabla F(W_{t-1})-M_t\|_\F + \frac{L}{2}\eta^2\breve{G}^2  \nonumber \\
		\leq & F(W_{t-1})  - \frac{\eta}{2}\|\nabla F(W_{t-1}) \|_\F + \eta\|\nabla F(W_{t-1})-M_t\|_\F + \frac{L}{2}\eta^2\breve{G}^2,
	\end{align}
	where the above inequality $(i)$ is due to $-\eta\langle\nabla F(W_{t-1}), M_t \rangle = - \frac{\eta}{2}\|\nabla F(W_{t-1})\|_\F- \frac{\eta}{2}\|M_t\|_\F + \frac{\eta}{2}\|\nabla F(W_{t-1})-M_t\|_\F$ and $- \eta\lambda\langle \nabla F(W_{t-1}), W_{t-1}\rangle \leq \eta\lambda\|\nabla F(W_{t-1})\|_\F \|W_{t-1}\|_\F$, and the above inequality $(ii)$ holds by
	$\|W_{t-1}\|_\F \leq t\eta\hat{G}$, and the above inequality $(iii)$ holds by $\lambda \leq \frac{1}{2(1+T)\eta\hat{G}}$.
	
	Then we rewrite the above inequality~(\ref{eq:c3}), and take expectation on this inequality as follows:
	\begin{align} \label{eq:c4}
		\E \|\nabla F(W_{t-1})\|_\F  \leq \E \big[ \frac{2(F(W_{t-1})-F(W_{t}))}{\eta} \big]  + 2 \E\|\nabla F(W_{t-1})-M_t\|_\F + L\eta\breve{G}^2.
	\end{align}
	
	According the above inequalities~(\ref{eq:c2}) and~(\ref{eq:c4}), since $r\geq 1$, we have for all $t\geq 1$
	\begin{align} \label{eq:c5}
		\E \|\nabla F(W_{t-1})\|_\F  \leq \E \big[ \frac{2(F(W_{t-1})-F(W_{t}))}{\eta} \big] + 4\sqrt{r}\E \|\nabla F(W_{t-1})-M_t\|_\F + L\eta\breve{G}^2.
	\end{align}
	
	Then we have
	\begin{align}
		& \frac{1}{T+1} \sum_{t=1}^{T+1} \E \|\nabla F(W_{t-1})\|_\F \nonumber \\
		& \leq \frac{1}{T+1} \sum_{t=1}^{T+1} \E \big[ \frac{2(F(W_{t-1})-F(W_{t}))}{\eta} \big] + \frac{4\sqrt{r}}{T+1} \sum_{t=1}^{T+1} \E \|\nabla F(W_{t-1})-M_t\|_\F + L\eta\breve{G}^2 \nonumber \\
		& \leq \frac{2(F(W_0)-F^*)}{T\eta}  + \frac{4\sqrt{r}}{T+1} \sum_{t=1}^{T+1}\E \|\nabla F(W_{t-1})-M_t\|_\F + L\eta\breve{G}^2 \nonumber \\
		& \leq \frac{2(F(W_0)-F^*)}{T\eta}  + 4\sqrt{r}\Big( \frac{\sigma}{\sqrt{T\beta}} +  \frac{\sqrt{2}}{\beta}L\eta\breve{G}  + \sqrt{\beta}\sigma \Big) + L\eta\breve{G}^2,
	\end{align}
	where the second inequality holds by Assumption~\ref{ass:b}, and the last inequality is due to
	Lemma~\ref{lem:mb2} and $\frac{1}{T+1} \sum_{t=1}^{T+1}\E \|\nabla F(W_{t-1})-M_t\|_\F=\frac{1}{T+1} \sum_{t=0}^{T}\E \|\nabla F(W_{t})-M_{t+1}\|_\F$.
	
	Since $\frac{1}{T+1} \sum_{t=0}^{T} \E \|\nabla F(W_t)\|_\F = \frac{1}{T+1} \sum_{t=1}^{T+1} \E \|\nabla F(W_{t-1})\|_\F$, we have
	\begin{align}
		\frac{1}{T+1} \sum_{t=0}^{T} \E \|\nabla F(W_t)\|_\F  \leq \frac{2(F(W_0)-F^*)}{T\eta}  + 4\sqrt{r}\Big( \frac{\sigma}{\sqrt{T\beta}} +  \frac{\sqrt{2}}{\beta}L\eta\breve{G}  + \sqrt{\beta}\sigma \Big) + L\eta\breve{G}^2.
	\end{align}
	
	Given $\eta=O(\frac{1}{T^{\frac{3}{4}}})$ and $\beta=\frac{1}{\sqrt{T}}$, we can obtain
	\begin{align}
		\frac{1}{T+1} \sum_{t=0}^T \E \|\nabla F(W_t)\|_\F
		& \leq \frac{2(F(W_0)-F^*)}{T\eta}  + 4\sqrt{r}\Big( \frac{\sigma}{\sqrt{T\beta}} +  \frac{\sqrt{2}}{\beta}L\eta\breve{G}  + \sqrt{\beta}\sigma \Big) + L\eta\breve{G}^2 \nonumber \\
		& = O(\frac{1}{T^{\frac{1}{4}}}) + O(\frac{1}{T^{\frac{1}{4}}}) + O(\frac{1}{T^{\frac{1}{4}}})+ O(\frac{1}{T^{\frac{1}{4}}}) + O(\frac{1}{T^{\frac{3}{4}}}) = O(\frac{1}{T^{\frac{1}{4}}}).
	\end{align}
	
\end{proof}

\section{Hyperparameter Settings and Search Ranges}
\label{app:hyperparams}

In this section, we report the key hyperparameters used in the main experiments and the ranges considered during tuning. For Muon-style optimizers, auxiliary parameters refer to parameters that are not eligible for matrix orthogonalization, such as embeddings, normalization parameters, biases, and task heads.
Note that our MiMuon algorithm uses the same  momentum-based gradient estimate as in Muon algorithm~\citep{jordanmuon} in the experiments, so our MiMuon algorithm use the momentum parameter $\mu$ instead of $\beta$. 

\begin{table}[t]
	\centering
	\caption{Hyperparameter search space for \textbf{Qwen3-0.6B}.}
	\label{tab:qwen_search_ranges}
	\setlength{\tabcolsep}{10pt}
	\renewcommand{\arraystretch}{1.16}
	\begin{tabularx}{0.95\textwidth}{>{\raggedright\arraybackslash}X>{\centering\arraybackslash}X}
		\toprule
		Learning rate & $\{2\mathrm{e}{-5}, 5\mathrm{e}{-5}, 1\mathrm{e}{-4}, 3\mathrm{e}{-4}, 5\mathrm{e}{-4}\}$ \\
		Muon LR & $\{0.010, 0.015, 0.020, 0.025\}$ \\
		Auxiliary LR & $\{5\mathrm{e}{-5}, 1\mathrm{e}{-4}, 3\mathrm{e}{-4}\}$ \\
		Weight decay & $\{0.05, 0.1, 0.2, 0.3\}$ \\
		Threshold $\tau$ & $\{0.002, 0.005, 0.01, 0.02\}$ \\
		MuSGD weights $(w_{\mu}, w_{\mathrm{sgd}})$ & $\{(0.2,1.0), (0.7,0.4), (0.8,0.2)\}$ \\
		\bottomrule
	\end{tabularx}
\end{table}

\begin{table}[t]
	\centering
	\caption{Hyperparameter search space for \textbf{YOLO26m}.}
	\label{tab:yolo_search_ranges}
	\setlength{\tabcolsep}{10pt}
	\renewcommand{\arraystretch}{1.16}
	\begin{tabularx}{0.95\textwidth}{>{\raggedright\arraybackslash}X>{\centering\arraybackslash}X}
		\toprule
		Learning rate & $\{1\mathrm{e}{-4}, 3\mathrm{e}{-4}, 1\mathrm{e}{-3}, 3\mathrm{e}{-3}\}$ \\
		Muon LR & $\{0.010, 0.015, 0.020, 0.025\}$ \\
		Auxiliary LR & $\{5\mathrm{e}{-4}, 1\mathrm{e}{-3}, 2\mathrm{e}{-3}\}$ \\
		Weight decay & $\{1\mathrm{e}{-4}, 5\mathrm{e}{-4}, 1\mathrm{e}{-3}, 5\mathrm{e}{-3}\}$ \\
		Threshold $\tau$ & $\{0.002, 0.005, 0.01, 0.02\}$ \\
		MuSGD weights $(w_{\mu}, w_{\mathrm{sgd}})$ & $\{(0.2,1.0), (0.7,0.4), (0.8,0.2)\}$ \\
		\bottomrule
	\end{tabularx}
\end{table}

\begin{table*}[t]
	\centering
	\caption{Selected hyperparameters and setup for \textbf{Qwen3-0.6B}.}
	\label{tab:qwen_hparams}
	\resizebox{0.92\textwidth}{!}{
	\begin{tabular}{lcccccc}
		\toprule
		& AdamW & Lion & Shampoo & Muon & MuSGD & MiMuon \\
		\midrule
		Main learning rate & $1\mathrm{e}{-4}$ & $5\mathrm{e}{-5}$ & $3\mathrm{e}{-4}$ & 0.02 & 0.015 & 0.02 \\
		Auxiliary learning rate & -- & -- & -- & $1\mathrm{e}{-4}$ & $3\mathrm{e}{-4}$ & $1\mathrm{e}{-4}$ \\
		Weight decay & 0.1 & 0.2 & 0.1 & 0.1 & 0.05 & 0.08 \\
		Momentum parameters & $(0.9,0.95)$ & $(0.95,0.98)$ & $(0.9,0.95)$ & $\mu=0.95$ & $\mu=0.95$ & $\mu=0.95$ \\
		Newton--Schulz steps & -- & -- & -- & 5 & -- & 5 \\
		Shampoo max dimension & -- & -- & 768 & -- & -- & -- \\
		MuSGD $w_{\mu}$ & -- & -- & -- & -- & 0.7 & -- \\
		MuSGD $w_{\mathrm{sgd}}$ & -- & -- & -- & -- & 0.4 & -- \\
		MiMuon threshold $\tau$ & -- & -- & -- & -- & -- & 0.005 \\
		\midrule
		Dataset & \multicolumn{6}{c}{WikiText-103} \\
		Sequence length & \multicolumn{6}{c}{1024} \\
		Global batch size & \multicolumn{6}{c}{16 sequences} \\
		Warmup steps & \multicolumn{6}{c}{300} \\
		LR scheduler & \multicolumn{6}{c}{cosine} \\
		Initialization & \multicolumn{6}{c}{Scratch} \\
		\bottomrule
	\end{tabular}
}
\end{table*}

\begin{table*}[t]
	\centering
	\caption{Selected hyperparameters and setup for \textbf{YOLO26m}.}
	\label{tab:yolo_hparams}
    \resizebox{0.92\textwidth}{!}{
	\begin{tabular}{lcccccc}
		\toprule
		& AdamW & Lion & Shampoo & Muon & MuSGD & MiMuon \\
		\midrule
		Main learning rate & $1\mathrm{e}{-3}$ & $3\mathrm{e}{-4}$ & $1\mathrm{e}{-3}$ & 0.02 & 0.02 & 0.02 \\
		Auxiliary learning rate & -- & -- & -- & $1\mathrm{e}{-3}$ & $2\mathrm{e}{-3}$ & $1\mathrm{e}{-3}$ \\
		Weight decay & $5\mathrm{e}{-4}$ & $1\mathrm{e}{-3}$ & $1\mathrm{e}{-4}$ & $5\mathrm{e}{-4}$ & $1\mathrm{e}{-3}$ & $1\mathrm{e}{-3}$ \\
		Momentum parameters & $(0.9,0.999)$ & $(0.9,0.99)$ & -- & $\mu=0.95$ & $\mu=0.95$ & $\mu=0.95$ \\
		Newton--Schulz steps & -- & -- & -- & 5 & -- & 5 \\
		Shampoo update frequency & -- & -- & 20 & -- & -- & -- \\
		MuSGD $w_{\mu}$ & -- & -- & -- & -- & 0.7 & -- \\
		MuSGD $w_{\mathrm{sgd}}$ & -- & -- & -- & -- & 0.4 & -- \\
		MiMuon threshold $\tau$ & -- & -- & -- & -- & -- & 0.01 \\
		\midrule
		Dataset & \multicolumn{6}{c}{Pascal VOC} \\
		Image size & \multicolumn{6}{c}{640} \\
		Global batch size & \multicolumn{6}{c}{32 images} \\
		Warmup steps & \multicolumn{6}{c}{500} \\
		LR scheduler & \multicolumn{6}{c}{cosine} \\
		Initialization & \multicolumn{6}{c}{Scratch} \\
		\bottomrule
	\end{tabular}
}
\end{table*}

\end{document}